\pdfoutput=1
\documentclass[11pt]{article}

\usepackage[final]{acl}

\usepackage{times}
\usepackage{latexsym}
\usepackage{blindtext}
\usepackage{titlesec}
\usepackage{microtype}
\usepackage{graphicx}
\usepackage{subfigure}
\usepackage{booktabs} % for professional tables
\usepackage{colortbl}
\usepackage{multirow}
\usepackage{amsmath}
\usepackage{amssymb}
\usepackage{mathtools}
\usepackage{amsthm}
\usepackage{geometry}
\usepackage{longtable}
\usepackage{minitoc}
\usepackage{array}
\usepackage{fontawesome5} % For GitHub icon

\theoremstyle{plain} % This is the default style
\newtheorem{theorem}{Theorem}
\usepackage[T1]{fontenc}
\usepackage[utf8]{inputenc}

\usepackage{inconsolata}

\usepackage{subcaption}

\usepackage{algorithm2e}  % For writing algorithms
\usepackage{algorithmic}
\usepackage{wrapfig}

\title{Predicting Through Generation: Why Generation Is Better for Prediction}

\author{
\textbf{Md Kowsher}\textsuperscript{1},
\textbf{Nusrat Jahan Prottasha}\textsuperscript{1},
\textbf{Prakash Bhat}\textsuperscript{2}, % Closing brace added here
\textbf{Chun-Nam Yu}\textsuperscript{3}, \\
\textbf{Mojtaba Soltanalian}\textsuperscript{4},
\textbf{Ivan Garibay}\textsuperscript{1},
\textbf{Ozlem Garibay}\textsuperscript{1},
\textbf{Chen Chen}\textsuperscript{1},
\textbf{Niloofar Yousefi}\textsuperscript{1}, \\
\textsuperscript{1}University of Central Florida, USA
\textsuperscript{2}DotStar Inc, USA \\
\textsuperscript{3}Nokia Bell Labs, USA
\textsuperscript{4}University of Illinois Chicago, USA\\
\faGithub~\href{https://github.com/Kowsher/PredGen}{\textcolor{red}{\texttt{github.com/Kowsher/PredGen}}}
}

\begin{document}
\maketitle
\begin{abstract}
This paper argues that generating output tokens is more effective than using pooled representations for prediction tasks because token-level generation retains more mutual information. Since LLMs are trained on massive text corpora using next-token prediction, generation aligns naturally with their learned behavior. Using the \textit{Data Processing Inequality (DPI)}, we provide both theoretical and empirical evidence supporting this claim. However, autoregressive models face two key challenges when used for prediction: (1) \textit{exposure bias}, where the model sees ground-truth tokens during training but relies on its own predictions during inference, leading to errors, and (2) \textit{format mismatch}, where discrete tokens do not always align with the task’s required output structure. To address these challenges, we introduce \textbf{PredGen (Predicting Through Generating)}, an end-to-end framework that (i) uses \textit{scheduled sampling} to reduce exposure bias, and (ii) introduces a \textit{task adapter} to convert the generated tokens into structured outputs. Additionally, we introduce \textit{Writer-Director Alignment Loss (WDAL)}, which ensures consistency between token generation and final task predictions, improving both text coherence and numerical accuracy. We evaluate \textbf{PredGen} on multiple classification and regression benchmarks. Our results show that \textbf{PredGen} consistently outperforms standard baselines, demonstrating its effectiveness in structured prediction tasks.
\end{abstract}

\begin{figure*}[!t]
%\begin{wrapfigure}{r}{0.7\textwidth}
%\captionsetup{font=footnotesize}

 \begin{center}
      \includegraphics[width=1.00\linewidth]{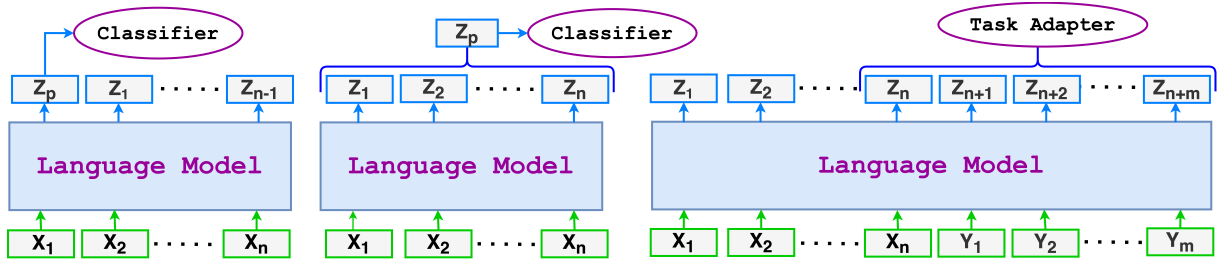}

   \end{center}
%\end{wrapfigure}
\caption{Comparison of different prediction methods using a language model. 
(Left) The traditional approach where a pooled representation \(\mathbf{Z_p}\) is passed to a classifier for prediction. 
(Middle) A similar method where \(\mathbf{Z_p}\) is extracted from the hidden states and used for classification. 
(Right) The generative approach, where the model generates additional tokens \(\mathbf{Y_1}, \mathbf{Y_2}, ..., \mathbf{Y_{m}}\), and their hidden states are processed by a task adapter for prediction. This method retains more task-relevant information by using token-level generation.}

\label{fig:method_cmp}
\end{figure*}

\section{Introduction}

Large Language Models (LLMs) have significantly advanced natural language processing (NLP), demonstrating strong performance across various tasks such as text completion \cite{kenton2019bert}, machine translation \cite{vaswani2017attention, wu2016google}, summarization \cite{lewis2019bart, zhang2020pegasus}, and question answering \cite{rajpurkar2016squad, yang2019xlnet}. By training on massive text corpora, these models learn contextual embeddings that capture rich semantic information, enabling them to generalize across a wide range of applications \cite{bommasani2021opportunities}. Beyond traditional NLP tasks, LLMs are increasingly used for predictive tasks, such as classification \cite{liu2019roberta}, regression \cite{raffel2020exploring}, and reasoning \cite{wei2022chain}, where they map input sequences to structured outputs.

A key strength of LLMs is their ability to perform tasks in a \emph{zero-shot} or \emph{few-shot} setting \cite{brown2020language}. By conditioning on a few examples or carefully crafted prompts, LLMs can generalize to new tasks without explicit fine-tuning \cite{kojima2022large}. However, while prompting is flexible, it lacks precision, particularly for tasks requiring structured outputs, such as numerical reasoning \cite{lewkowycz2022solving}. To improve accuracy, fine-tuning is often performed by training a prediction head on pooled representations (e.g., [CLS] tokens or mean-pooled embeddings) \cite{kenton2019bert} (Figure~\ref{fig:method_cmp}-Left \& Middle). However, pooling discards positional and sequential information, limiting the model’s ability to capture fine-grained dependencies \cite{huang2024scalingnote, oh2022don}.

We argue that generation-based training, where an LLM is fine-tuned to produce task outputs as token sequences, preserves richer information than classification-based approaches  (Figure~\ref{fig:method_cmp}-Right). Since LLMs are originally trained using next-token prediction on large corpora, generation aligns naturally with their learning paradigm. Switching to pooled classification may not fully use their pre-training knowledge, leading to weaker transfer learning. Our experiments show that token-level generation retains more mutual information than traditional prediction methods. Using the Data Processing Inequality (DPI)\cite{beaudry2011intuitive}, we theoretically prove that generating tokens preserves strictly more information with the target output than pooling-based representations, preventing irreversible information loss.

Despite the benefits of generation, two key challenges arise. First, many tasks require structured outputs, but generative models produce discrete tokens. For example, in a Semantic Textual Similarity (STS-B) task, the model must output a similarity score like 0.75. Representing this as separate tokens \([0, ., 7, 5]\) can introduce ambiguities—generating "0.76" or "1.75" may lead to similar token-level losses, even though 0.76 is numerically much closer to 0.75 than 1.75. This makes fine-grained numerical learning difficult. Second, exposure bias occurs because, during training, the model always conditions on ground-truth tokens, but during inference, it must rely on its own previously generated outputs. Small errors in the early steps can accumulate, leading to compounding inaccuracies.

To address these issues, we introduce  PredGen (Predicting Through Generating), an end-to-end framework that fine-tunes LLMs for supervised prediction tasks. PredGen treats the target output as a sequence of tokens and incorporates scheduled sampling \cite{bengio2015scheduled} to mitigate exposure bias. Additionally, a task adapter transforms discrete token outputs into structured predictions, ensuring both numerical precision and task-specific formatting.

Furthermore, we introduce a specialized loss function, termed \textit{Writer-Director Alignment Loss} (WDAL), designed to align token generation with task-specific predictions. In this framework, the \textit{writer} generates output tokens (analogous to drafting a film script), while the \textit{director} transforms these tokens into the required task format (comparable to producing a film from a script). WDAL ensures that the generated sequence maintains both textual coherence and numerical accuracy by effectively balancing token-level generation with task-specific objectives.

We evaluate PredGen on multiple regression and classification benchmarks, covering both numerical reasoning tasks (e.g., mathematical problem solving, similarity scoring) and high-level reasoning tasks (e.g., commonsense understanding). PredGen consistently outperforms baselines that use pooled representations or a standard generative approach, demonstrating its effectiveness across a wide range of structured prediction tasks. 

\textbf{Our contributions} are summarized as follows:
\begin{itemize}
    \item  We argue that generation is superior to traditional classifier-based prediction and provide both theoretical and empirical evidence to support this claim.
    \item  We introduce \textbf{PredGen}, a framework designed to address the key challenges in generative prediction.
    \item We propose a novel loss function, WDAL, that aligns token generation with final task predictions to ensure consistency between the generated sequences and structured outputs.
\end{itemize}

\section{Problem Formulation}

\paragraph{Prediction Using Language Models.}
Let \(\mathbf{X} = [\mathbf{X}_1, \mathbf{X}_2, \dots, \mathbf{X}_n]\) be an input sequence, and suppose we wish to predict a structured output \(\mathbf{P}\). For instance, if \(\mathbf{P} = 13.4\), we may represent it as a discrete token sequence \(\mathbf{Y} = [\text{'1'}, \text{'3'}, \text{'.'}, \text{'4'}]\). A pre-trained language model \(\mathcal{M}\) encodes \(\mathbf{X}\) into hidden states 
\[
\mathbf{Z} = [\mathbf{Z}_1, \mathbf{Z}_2, \dots, \mathbf{Z}_n],
\]
where \(\mathbf{Z}_i \in \mathbb{R}^d \) is the contextual embedding for token \(x_i\) with $d$ dimension. In standard prediction settings, we often \emph{pool} these hidden states into a single vector representation \(\mathbf{Z}_p  \in \mathbb{R}^d \). A classifier function \(f_{\mathrm{cls}}\bigl(\mathbf{Z}_p\bigr)\) then transforms this pooled representation into the final prediction \(\hat{\mathbf{P}}\), typically returning probabilities (for classification) or a real-valued score (for regression).

\paragraph{Reformulating Prediction as Token Generation:}  
We redefine classification as a sequence generation task. Given an input sequence \( \mathbf{X} \) and target sequence \( \mathbf{Y} \), the model is trained to generate \( \mathbf{Y} \) autoregressively, predicting one token at a time. The probability distribution of the target sequence is given by:
\[
P(\mathbf{Y} | \mathbf{X}; \theta) = \prod_{t=1}^m P(Y_t | \mathbf{X}, \mathbf{Y}_1, \dots, \mathbf{Y}_{t-1}; \theta),
\]
where \( \theta \) represents the learnable parameters of the model. This formulation allows the model to generate structured outputs while preserving sequential dependencies in the data.

\section{Why Generation is More Effective for Prediction}
\label{sec:generation-superior}

One main reason to treat prediction as a token-by-token generation process is that LLMs are originally trained on massive text corpora, which gives them strong generative skills. By generating each output token step by step, the model uses all the contextual clues it learned during pre-training and preserves more details that matter for the task. This property also explains why LLMs often succeed in zero-shot or few-shot scenarios.

In contrast, when we fine-tune an LLM by pooling all hidden states into a single representation, we rely on a \emph{deterministic} procedure. According to the Data Processing Inequality and \textbf{Theorem~\ref{thm:generator-better-than-pooling} (Mutual Information Decreases Under Deterministic Compression)}, this inevitably discards important information. Because token-level generation keeps more mutual information at each step, it can produce predictions that are both more precise and more faithful to the original context.

\begin{theorem}[]
\label{thm:generator-better-than-pooling}
Let $ \mathbf{X}$ be an input random variable, and let $ \mathbf{Z} \in \mathcal{Z}$ be the final hidden representation produced by a model (e.g., an LLM).  
Suppose $ \mathbf{Z_p} = g( \mathbf{Z})$ for some deterministic function $g : \mathcal{Z} \to \mathcal{W}$ (e.g., first-token pooling or mean pooling).  
Let $ \mathbf{Y}$ be the target random variable.  
Then the following holds:
\[
    I( \mathbf{Y} \; ; \;  \mathbf{Z}) \;\;\ge\;\; I( \mathbf{Y} \; ; \;  \mathbf{Z_p}).
\]

\end{theorem}

\textit{Proof Sketch:}  
The core idea behind the theorem relies on the \textit{DPI}, which states that applying a deterministic function to a random variable cannot increase the mutual information between the original variable and another variable. Since \(\mathbf{Z}_p\) is derived from \(\mathbf{Z}\) through a deterministic process, it follows that \(\mathbf{Z}_p\) contains less information than \(\mathbf{Z}\) (or at most, the same amount), meaning that conditioning on \(\mathbf{Z}_p\) cannot reduce uncertainty about \(\mathbf{Y}\) more than conditioning on \(\mathbf{Z}\). This leads to the following relationship between the conditional entropies:
\[
H(\mathbf{Y} \mid \mathbf{Z}) \leq H(\mathbf{Y} \mid \mathbf{Z}_p).
\]
which leads to the following inequality in the mutual information:
\[
I(\mathbf{Y}; \mathbf{Z}_p) \leq I(\mathbf{Y}; \mathbf{Z}),
\]
A detailed proof is provided in Appendix~\ref{app_sec:proof_theorem_1}.

\begin{figure}[!t]
    \centering
    \includegraphics[width=1.0\linewidth]{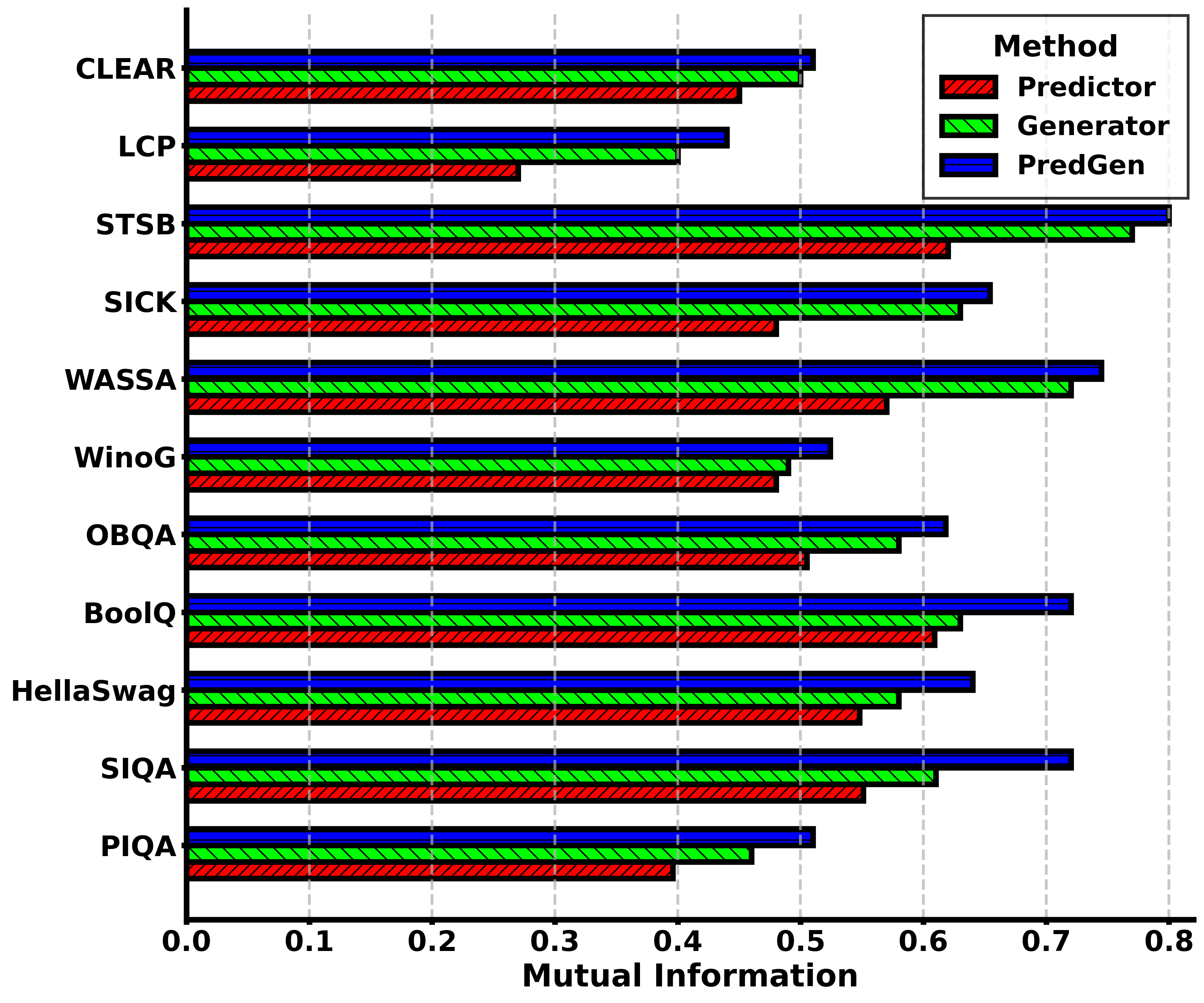}
  
    \caption{Comparison of mutual information estimates for Predictor, Generator, and PredGen across multiple datasets. PredGen consistently retains higher mutual information, supporting the theoretical claim that token-level generation preserves richer task-relevant information than pooled representations.}
\label{fig:MI}
\end{figure}

\begin{figure}[!t]
    \centering
    \includegraphics[width=1.0\linewidth]{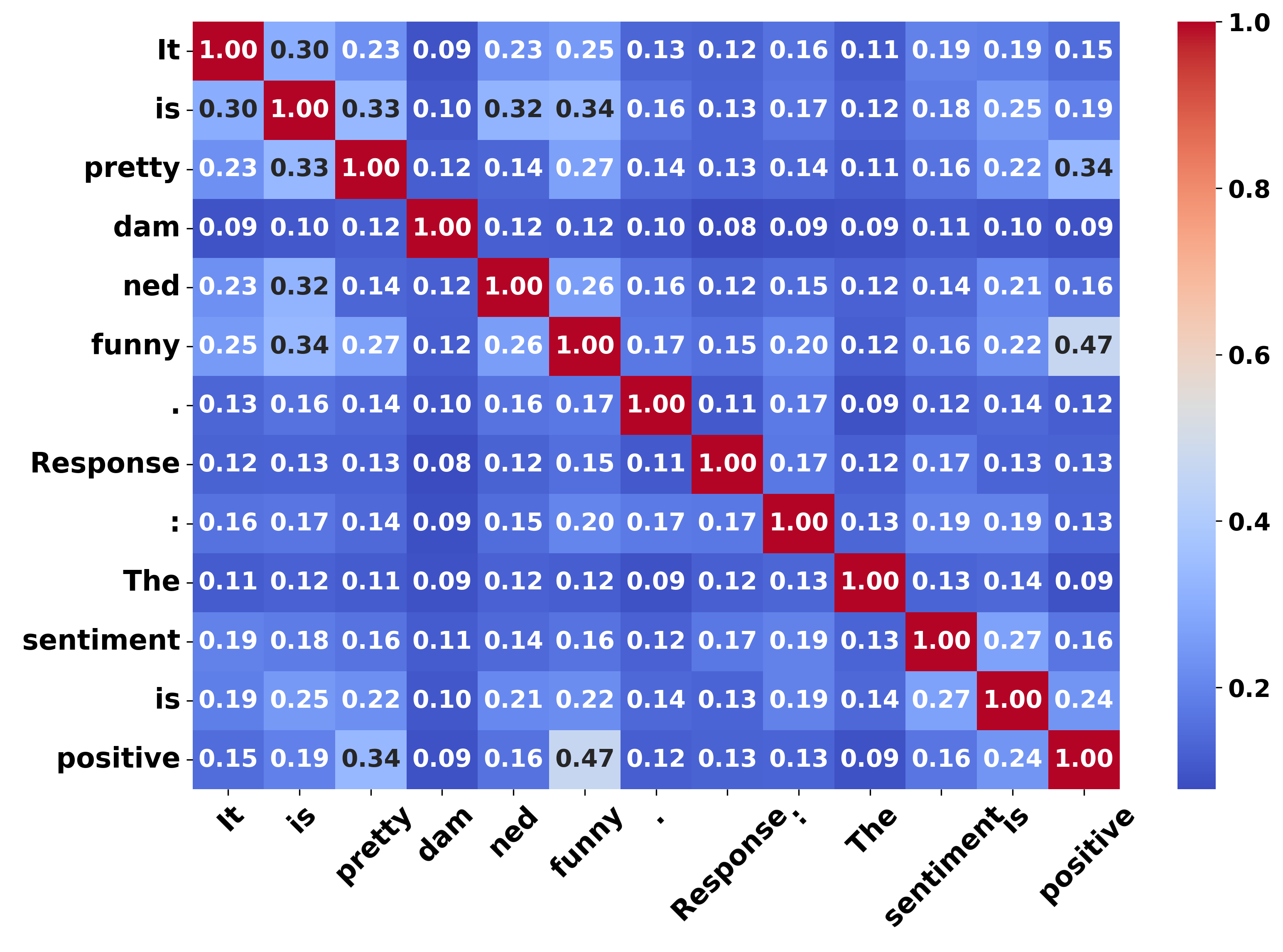}

    \caption{Token-wise mutual information on SST-2 \cite{socher2013recursive}. The predicted token \textit{"positive"} shows high MI with sentiment-related tokens like \textit{"funny"} (0.47) and \textit{"pretty"} (0.34), highlighting strong contextual dependencies.}

\label{fig:MI_token}
\end{figure}

\paragraph{Empirical Evidence:} 
To empirically validate Theorem \ref{thm:generator-better-than-pooling}, we estimate mutual information using the MINE method proposed by \citet{belghazi2018mine}. We train separate models for the Predictor, Generator, and PredGen approaches and extract the pooled representation $\mathbf{Z_p}$ for the Predictor and $\mathbf{Z}$ for the Generator and PredGen across all training and testing data. 

To simplify the computational cost of using all token representations in \(\mathbf{Z}\), we apply Principal Component Analysis (PCA) \cite{mackiewicz1993principal} to reduce the original \(n \times d\) representation down to a \(2 \times d\) space. Let $ \mathbf{Z_r}$ be the reduced representation $\mathbf{Z_r} = \mathrm{PCA}(\mathbf{Z})$.

Next, we estimate mutual information using a two-layer neural network with MINE \cite{belghazi2018mine}, which uses a neural variational method to learn a lower bound on mutual information. When evaluating the \emph{predictor}, we feed \(\mathbf{Z_p}\) as input; when evaluating \emph{generation}, we use the reduced representation \(\mathbf{Z_r}\). The estimation function is:
\begin{multline*}
   \mathcal{I}(\mathbf{Y}; \mathbf{Z}) = \sup_{\theta \in \Theta} \mathbb{E}_{p_{(\mathbf{YZ})}}[T_\theta(\mathbf{Y}, \mathbf{Z})] \\
   - \log \mathbb{E}_{p_{(\mathbf{Y})} p_{(\mathbf{Z})}}[e^{T_\theta(\mathbf{Y}, \mathbf{Z})}] 
\end{multline*}
where $T_\theta$ is a trainable function parameterized by $\theta$. 

Figure \ref{fig:MI} compares mutual information estimates on the test set, showing that PredGen consistently retains more information than both the Predictor and Generator, supporting our theoretical claims.

Additionally, Figure \ref{fig:MI_token} provides a detailed token-wise mutual information analysis. We train the model on the SST-2 dataset and compute MI between token pairs. The predicted sentiment token \textit{"positive"} shows high MI with sentiment-relevant words such as \textit{"funny"} (0.47) and \textit{"pretty"} (0.34), indicating that PredGen effectively captures contextual dependencies. For additional details on mutual information in regression and classification, refer to Appendix Section~\ref{append:MI}. The experimental setup is provided in Section~\ref{sec:exp}.

\section{PredGen}

When using generation to perform prediction, one primary challenge is \emph{exposure bias}. In a typical autoregressive setup, the model is trained to predict the next token based on all previous \emph{ground-truth} tokens. Specifically, at time step \(t\), the model receives \([\mathbf{X}_1, \mathbf{X}_2, \dots, \mathbf{X}_n, \mathbf{Y}_1, \mathbf{Y}_2, \dots, \mathbf{Y}_{t-1}]\) and produces \(y_t\). However, during inference, the model must rely on its own previously \emph{generated} tokens \(\tilde{\mathbf{Y}}_1, \tilde{\mathbf{Y}}_2, \dots, \tilde{\mathbf{Y}}_{t-1}\)\textemdash not the ground-truth sequence. This mismatch means the model never learns to correct its own mistakes, since it always conditions on true tokens during training but must condition on its own (potentially flawed) outputs at test time. Consequently, small errors can accumulate and lead to compounding inaccuracies.

To address exposure bias during autoregressive training, we apply sequence level \emph{scheduled sampling} \cite{bengio2015scheduled}. 
\[
\tilde{\mathbf{Y}} \;=\; 
\begin{cases} 
      \mathbf{Y} & \text{with probability } (1 - p), \\
      \tilde{\mathbf{Y}} & \text{with probability } p,
\end{cases}
\]
where \(\mathbf{Y}\) is the ground-truth tokens and \(\tilde{\mathbf{Y}}_{t-1}\) is the generated predictions' tokens. The parameter \(p\) gradually increases during training, shifting from ground truth to self-conditioning (predictions). This helps the model learn to correct errors arising from its own outputs, thus reducing exposure bias.

Another challenge is that generative models produce \emph{discrete} tokens, whereas certain tasks (e.g., regression) require \emph{continuous} values. To address this, we introduce a transformation step (\emph{Task Adapter}) that maps the generated token sequence into the final prediction form. Concretely, let 
\[
\mathbf{Z}[n:n+m] = [\mathbf{Z}_{n}, \mathbf{Z}_{n+1}, \cdots \mathbf{Z}_{n+m}]
\]

Here, \(\mathbf{Z}[n:n+m]\) denotes the hidden representations for the generated tokens, where \(n\) is the length of the input \(\mathbf{X}\) and \(m\) is the number of generated tokens. Now We define a \emph{task adapter} \(\mathcal{T}\) that transforms \(\tilde{\mathbf{Y}}\) into the desired output:
\[
   \hat{\mathbf{P}} = \mathcal{T}(\mathbf{Z}[n:n+m]).
\]
For example, \(\mathcal{T}\) could convert a sequence of digit tokens into a real-valued number for regression or map generated tokens to a categorical label (e.g., 0, 1, 2, \dots) for classification. This ensures that discrete outputs from the generator can accommodate both continuous and structured predictions.

\paragraph{Writer-Director Alignment Loss (WDAL):}
\label{sec:wdal-main}

In generative prediction, ensuring that the generated token sequence aligns with the final structured output is crucial. The \textit{writer} (generator) is responsible for generating tokens, while the \textit{director} (task adapter) transforms them into the required task format. If these two components are not well-coordinated, errors in generation can propagate to the final prediction, reducing accuracy. To address this issue, we introduce a novel loss function, \textbf{WDAL}, which optimizes both components together to improve prediction quality.

WDAL consists of two primary loss terms: the \textit{writer loss} \(L_W\), which measures the generation error using cross-entropy loss, and the \textit{director loss} \(L_D\), which quantifies the error in the final prediction. A natural way to combine these losses is through their product, \( L_{\text{WDAL}} = L_W \cdot L_D \), ensuring that both the generation and task-specific transformation contribute to the optimization. However, this formulation can lead to numerical instability when the losses differ significantly in scale. To address this, we apply a log-sum-exp trick, leading to the final formulation:
\begin{multline*}\label{eq:wdal_final}
  L_{\text{WDAL}} = \max\!\bigl(L_W^2, L_D^2\bigr)\,\\\exp\!\Bigl(-\Bigl|\log L_W - \log L_D\Bigr|\Bigr).
\end{multline*}
The first term, \(\max(L_W^2, L_D^2)\), serves as an \textit{authority component}, prioritizing the larger loss so that optimization focuses on the component with higher error. The second term, \(\exp\!\bigl(-|\log L_W - \log L_D|\bigr)\), acts as an \textit{alignment penalty}, ensuring that both losses remain balanced. When one loss is significantly higher than the other, the penalty reduces, keeping the overall loss high and encouraging better coordination between the writer and director. A complete derivation of WDAL is provided in Appendix~\ref{sec:wdal-appendix}.

\section{Experiments} \label{sec:exp}
 We compare \textsc{PredGen} with two baselines: the traditional \textbf{Predictor}, which uses pooled hidden representations followed by classification, and the standard \textbf{Generator}, which directly generates output tokens without additional transformations.

To efficiently fine-tune large language models, we employ several PEFT techniques, including \textbf{LoRA} \cite{hu2021lora}, \textbf{AdaLoRA} \cite{zhang2023adalora}, \textbf{RoCoFT} \cite{kowsher2024rocoft}, and \textbf{DoRA} \cite{liu2024dora}. These methods allow us to adapt large models with fewer parameters, reducing computational costs while maintaining strong performance.

For classification, we evaluate on the following datasets: \texttt{BoolQ} \cite{clark2019boolq}, \texttt{PIQA} \cite{bisk2020piqa}, \texttt{SIQA} \cite{sap2019socialiqa}, \texttt{HellaSwag} \cite{zellers2019hellaswag}, \texttt{WinoGrande} \cite{sakaguchi2021winogrande}, \texttt{ARC-e} \cite{clark2018think}, \texttt{ARC-c} \cite{clark2018think}, and \texttt{OBQA} \cite{mihaylov2018can}. The reported metric for classification tasks is accuracy.

For regression, we use the following datasets: \texttt{WASSA} \cite{vinayakumar2017deepcybernet}, \texttt{SICK} \cite{marelli-etal-2014-sick}, \texttt{STSB} \cite{cer2017semeval}, \texttt{LCP} \cite{shardlow2020complex}, \texttt{CLEAR} \cite{crossley2023large}, and \texttt{Humicroedit} \cite{hossain2019president}. The reported metrics for regression tasks are Mean Squared Error (MSE) and Mean Absolute Error (MAE).

Details about the datasets, implementation details, and hyper-parameters  are provided in Appendix \ref{sec:dataset},  \ref{sec:impl}, and Table \ref{tab:hyperparameters} respectively.

\begin{table*}[ht]
\centering
\scalebox{.680}{
\begin{tabular}{l|l|l|c|c|c|c|c|c|c|c|c}
\hline
\rowcolor{gray!20}
\textbf{Model}     & \textbf{PEFT}    & \textbf{Method}    & \textbf{BoolQ} & \textbf{PIQA} & \textbf{SIQA} & \textbf{HellaSwag} & \textbf{WinoGrande} & \textbf{ARC-e} & \textbf{ARC-c} & \textbf{OBQA} & \textbf{Avg.}\\ \hline
\multirow{12}{*}{\textbf{Llama2-7B}} 
                   & \cellcolor[HTML]{EAF3FA}  LoRA             & \cellcolor[HTML]{EAF3FA}Predictor          & \cellcolor[HTML]{EAF3FA}66.29         & \cellcolor[HTML]{EAF3FA}81.11         & \cellcolor[HTML]{EAF3FA}78.95          & \cellcolor[HTML]{EAF3FA}88.53              & \cellcolor[HTML]{EAF3FA}70.49               & \cellcolor[HTML]{EAF3FA}75.27          & \cellcolor[HTML]{EAF3FA}53.9           & \cellcolor[HTML]{EAF3FA}73.39  & \cellcolor[HTML]{EAF3FA}73.49      \\ 
                   & \cellcolor[HTML]{EAF3FA}                  &  \cellcolor[HTML]{EAF3FA}Generator          &  \cellcolor[HTML]{EAF3FA}68.09          & \cellcolor[HTML]{EAF3FA}80.37         &  \cellcolor[HTML]{EAF3FA}77.15         &  \cellcolor[HTML]{EAF3FA}90.86              & \cellcolor[HTML]{EAF3FA}77.54               &  \cellcolor[HTML]{EAF3FA}79.54         &  \cellcolor[HTML]{EAF3FA}60.55          &  \cellcolor[HTML]{EAF3FA}78.93 &  \cellcolor[HTML]{EAF3FA}76.63        \\ 
                   & \cellcolor[HTML]{D0E7F7}                  & \cellcolor[HTML]{D0E7F7}PredGen         & \cellcolor[HTML]{D0E7F7}\textbf{73.82}          & \cellcolor[HTML]{D0E7F7}82.76         & \cellcolor[HTML]{D0E7F7}79.87         & \cellcolor[HTML]{D0E7F7}93.14              & \cellcolor[HTML]{D0E7F7}83.21               & \cellcolor[HTML]{D0E7F7}84.79          & \cellcolor[HTML]{D0E7F7}\textbf{60.86}          & \cellcolor[HTML]{D0E7F7}78.93 & \cellcolor[HTML]{D0E7F7}79.67        \\ 
                    & \cellcolor[HTML]{EAF3FA}    AdaLoRA           & \cellcolor[HTML]{EAF3FA}    Predictor           & \cellcolor[HTML]{EAF3FA}    65.22           & \cellcolor[HTML]{EAF3FA}    80.91          & \cellcolor[HTML]{EAF3FA}    78.88         & \cellcolor[HTML]{EAF3FA}    89.33              & \cellcolor[HTML]{EAF3FA}    70.13                & \cellcolor[HTML]{EAF3FA}    75.39           & \cellcolor[HTML]{EAF3FA}    54.29           & \cellcolor[HTML]{EAF3FA}    73.81    & \cellcolor[HTML]{EAF3FA}    73.50        \\ 
                    & \cellcolor[HTML]{EAF3FA}                    & \cellcolor[HTML]{EAF3FA}    Generator           & \cellcolor[HTML]{EAF3FA}    70.03          & \cellcolor[HTML]{EAF3FA}    80.69         & \cellcolor[HTML]{EAF3FA}    77.06          & \cellcolor[HTML]{EAF3FA}    90.85              & \cellcolor[HTML]{EAF3FA}    76.47                & \cellcolor[HTML]{EAF3FA}    79.50          & \cellcolor[HTML]{EAF3FA}    59.30          & \cellcolor[HTML]{EAF3FA}    74.22  & \cellcolor[HTML]{EAF3FA}     76.02      \\ 
                    & \cellcolor[HTML]{D0E7F7}      & \cellcolor[HTML]{D0E7F7}    PredGen            & \cellcolor[HTML]{D0E7F7}    72.45           & \cellcolor[HTML]{D0E7F7}    84.54          & \cellcolor[HTML]{D0E7F7}    \textbf{80.42}         & \cellcolor[HTML]{D0E7F7}    93.19               & \cellcolor[HTML]{D0E7F7}    82.26                & \cellcolor[HTML]{D0E7F7}    84.80           & \cellcolor[HTML]{D0E7F7}    59.53           & \cellcolor[HTML]{D0E7F7}    79.11  & \cellcolor[HTML]{D0E7F7}    79.54        \\ 
                    & \cellcolor[HTML]{EAF3FA}    RoCoFT            & \cellcolor[HTML]{EAF3FA}    Predictor           & \cellcolor[HTML]{EAF3FA}    66.48           & \cellcolor[HTML]{EAF3FA}    81.53         & \cellcolor[HTML]{EAF3FA}    79.85          & \cellcolor[HTML]{EAF3FA}    89.24               & \cellcolor[HTML]{EAF3FA}    68.84                & \cellcolor[HTML]{EAF3FA}    76.85           & \cellcolor[HTML]{EAF3FA}    54.38           & \cellcolor[HTML]{EAF3FA}    73.38   & \cellcolor[HTML]{EAF3FA}    73.82       \\ 
                    & \cellcolor[HTML]{EAF3FA}              & \cellcolor[HTML]{EAF3FA}    Generator           & \cellcolor[HTML]{EAF3FA}    69.36           & \cellcolor[HTML]{EAF3FA}    80.08          & \cellcolor[HTML]{EAF3FA}    77.99          & \cellcolor[HTML]{EAF3FA}    89.46               & \cellcolor[HTML]{EAF3FA}    77.41               & \cellcolor[HTML]{EAF3FA}    79.46           & \cellcolor[HTML]{EAF3FA}    59.09           & \cellcolor[HTML]{EAF3FA}    76.90   & \cellcolor[HTML]{EAF3FA}    76.22       \\ 
                    & \cellcolor[HTML]{D0E7F7}              & \cellcolor[HTML]{D0E7F7}    PredGen            & \cellcolor[HTML]{D0E7F7}    73.62           & \cellcolor[HTML]{D0E7F7}    84.32          & \cellcolor[HTML]{D0E7F7}    79.65          & \cellcolor[HTML]{D0E7F7}    92.64               & \cellcolor[HTML]{D0E7F7}    \textbf{83.83}                 & \cellcolor[HTML]{D0E7F7}    84.67           & \cellcolor[HTML]{D0E7F7}     60.81           & \cellcolor[HTML]{D0E7F7}    \textbf{80.20}  & \cellcolor[HTML]{D0E7F7}    79.97        \\ 
                    & \cellcolor[HTML]{EAF3FA}      DoRA           & \cellcolor[HTML]{EAF3FA}    Predictor           & \cellcolor[HTML]{EAF3FA}    66.23           & \cellcolor[HTML]{EAF3FA}    82.24  & \cellcolor[HTML]{EAF3FA}    78.83          & \cellcolor[HTML]{EAF3FA}    88.21               & \cellcolor[HTML]{EAF3FA}    71.36                & \cellcolor[HTML]{EAF3FA}    73.78           & \cellcolor[HTML]{EAF3FA}    54.22           & \cellcolor[HTML]{EAF3FA}    75.48   & \cellcolor[HTML]{EAF3FA}     73.79      \\ 
                    & \cellcolor[HTML]{EAF3FA}                      & \cellcolor[HTML]{EAF3FA}    Generator           & \cellcolor[HTML]{EAF3FA}    69.56           & \cellcolor[HTML]{EAF3FA}    80.33          & \cellcolor[HTML]{EAF3FA}    77.09          & \cellcolor[HTML]{EAF3FA}    90.10               & \cellcolor[HTML]{EAF3FA}    76.69                & \cellcolor[HTML]{EAF3FA}    79.66           & \cellcolor[HTML]{EAF3FA}    59.05           & \cellcolor[HTML]{EAF3FA}    76.96  & \cellcolor[HTML]{EAF3FA}    76.18          \\ 
                    & \cellcolor[HTML]{D0E7F7}         & \cellcolor[HTML]{D0E7F7}    PredGen            & \cellcolor[HTML]{D0E7F7}    73.45           & \cellcolor[HTML]{D0E7F7}    \textbf{84.73}          & \cellcolor[HTML]{D0E7F7}    80.11          & \cellcolor[HTML]{D0E7F7}    \textbf{93.28}               & \cellcolor[HTML]{D0E7F7}    83.80                & \cellcolor[HTML]{D0E7F7}    \textbf{84.99}           & \cellcolor[HTML]{D0E7F7}    60.19           & \cellcolor[HTML]{D0E7F7}    79.89  & \cellcolor[HTML]{D0E7F7}     \textbf{80.06}       \\ 
\hline
\multirow{12}{*}{\textbf{Llama2-13B}} 
                    & \cellcolor[HTML]{F8EDEB}     LoRA              & \cellcolor[HTML]{F8EDEB}    Predictor           & \cellcolor[HTML]{F8EDEB}    68.37          & \cellcolor[HTML]{F8EDEB}    83.42          & \cellcolor[HTML]{F8EDEB}    81.34    & \cellcolor[HTML]{F8EDEB}    91.54               & \cellcolor[HTML]{F8EDEB}    72.32                & \cellcolor[HTML]{F8EDEB}    79.24           & \cellcolor[HTML]{F8EDEB}    55.12           & \cellcolor[HTML]{F8EDEB}    78.23    & \cellcolor[HTML]{F8EDEB}     76.20     \\ 
                    & \cellcolor[HTML]{F8EDEB}    \                  & \cellcolor[HTML]{F8EDEB}    Generator           & \cellcolor[HTML]{F8EDEB}    71.19          & \cellcolor[HTML]{F8EDEB}    83.99          & \cellcolor[HTML]{F8EDEB}    81.15          & \cellcolor[HTML]{F8EDEB}    92.86               & \cellcolor[HTML]{F8EDEB}    83.24                 & \cellcolor[HTML]{F8EDEB}    83.35           & \cellcolor[HTML]{F8EDEB}    66.05           & \cellcolor[HTML]{F8EDEB}    81.37   & \cellcolor[HTML]{F8EDEB}    80.40       \\ 
                    & \cellcolor[HTML]{F3DBD7}        & \cellcolor[HTML]{F3DBD7}    PredGen            & \cellcolor[HTML]{F3DBD7}    73.43           & \cellcolor[HTML]{F3DBD7}    85.32          & \cellcolor[HTML]{F3DBD7}    \textbf{82.45}          & \cellcolor[HTML]{F3DBD7}    94.25               & \cellcolor[HTML]{F3DBD7}    85.82                & \cellcolor[HTML]{F3DBD7}    \textbf{86.79}           & \cellcolor[HTML]{F3DBD7}    68.24           & \cellcolor[HTML]{F3DBD7}    \textbf{85.41}  & \cellcolor[HTML]{F3DBD7}    82.71         \\ 
                    & \cellcolor[HTML]{F8EDEB}     AdaLoRA           & \cellcolor[HTML]{F8EDEB}    Predictor           & \cellcolor[HTML]{F8EDEB}    69.83           & \cellcolor[HTML]{F8EDEB}    84.38           & \cellcolor[HTML]{F8EDEB}    80.27          & \cellcolor[HTML]{F8EDEB}    90.19               & \cellcolor[HTML]{F8EDEB}    72.22                & \cellcolor[HTML]{F8EDEB}    78.77           & \cellcolor[HTML]{F8EDEB}    53.75           & \cellcolor[HTML]{F8EDEB}    79.56   & \cellcolor[HTML]{F8EDEB}    76.12       \\ 
                    & \cellcolor[HTML]{F8EDEB}                    & \cellcolor[HTML]{F8EDEB}    Generator           & \cellcolor[HTML]{F8EDEB}    71.71           & \cellcolor[HTML]{F8EDEB}    82.55          & \cellcolor[HTML]{F8EDEB}    81.88          & \cellcolor[HTML]{F8EDEB}    92.61               & \cellcolor[HTML]{F8EDEB}    83.01
                 & \cellcolor[HTML]{F8EDEB}    83.04           & \cellcolor[HTML]{F8EDEB}    67.33           & \cellcolor[HTML]{F8EDEB}    81.76  & \cellcolor[HTML]{F8EDEB}     80.49        \\ 
                    & \cellcolor[HTML]{F3DBD7}                & \cellcolor[HTML]{F3DBD7}    PredGen            & \cellcolor[HTML]{F3DBD7}    74.21          & \cellcolor[HTML]{F3DBD7}    85.99          & \cellcolor[HTML]{F3DBD7}    82.16          & \cellcolor[HTML]{F3DBD7}    94.51               & \cellcolor[HTML]{F3DBD7}    86.09                & \cellcolor[HTML]{F3DBD7}    86.42           & \cellcolor[HTML]{F3DBD7}    \textbf{69.73}           & \cellcolor[HTML]{F3DBD7}    84.98  & \cellcolor[HTML]{F3DBD7}      \textbf{83.01}       \\ 
                    & \cellcolor[HTML]{F8EDEB}     RoCoFT            & \cellcolor[HTML]{F8EDEB}    Predictor           & \cellcolor[HTML]{F8EDEB}    68.22           & \cellcolor[HTML]{F8EDEB}    82.90           & \cellcolor[HTML]{F8EDEB}    79.99           & \cellcolor[HTML]{F8EDEB}    91.28               & \cellcolor[HTML]{F8EDEB}    71.60                & \cellcolor[HTML]{F8EDEB}    79.21           & \cellcolor[HTML]{F8EDEB}    57.26           & \cellcolor[HTML]{F8EDEB}    78.56   & \cellcolor[HTML]{F8EDEB}    76.13         \\ 
                    & \cellcolor[HTML]{F8EDEB}                      & \cellcolor[HTML]{F8EDEB}    Generator           & \cellcolor[HTML]{F8EDEB}    71.44           & \cellcolor[HTML]{F8EDEB}    83.52          & \cellcolor[HTML]{F8EDEB}    79.50          & \cellcolor[HTML]{F8EDEB}    91.84               & \cellcolor[HTML]{F8EDEB}    83.20                & \cellcolor[HTML]{F8EDEB}    83.39           & \cellcolor[HTML]{F8EDEB}    68.06           & \cellcolor[HTML]{F8EDEB}    81.73   & \cellcolor[HTML]{F8EDEB}    80.33       \\ 
                    & \cellcolor[HTML]{F3DBD7}             & \cellcolor[HTML]{F3DBD7}    PredGen            & \cellcolor[HTML]{F3DBD7}    \textbf{74.27}           & \cellcolor[HTML]{F3DBD7}    \textbf{86.13}          & \cellcolor[HTML]{F3DBD7}    81.71          & \cellcolor[HTML]{F3DBD7}    \textbf{94.58}               & \cellcolor[HTML]{F3DBD7}    86.16                & \cellcolor[HTML]{F3DBD7}    85.79           & \cellcolor[HTML]{F3DBD7}    69.22           & \cellcolor[HTML]{F3DBD7}    85.29   & \cellcolor[HTML]{F3DBD7}    82.89       \\ 
                    & \cellcolor[HTML]{F8EDEB}     DoRA              & \cellcolor[HTML]{F8EDEB}    Predictor           & \cellcolor[HTML]{F8EDEB}    69.18           & \cellcolor[HTML]{F8EDEB}    83.20          & \cellcolor[HTML]{F8EDEB}    80.84          & \cellcolor[HTML]{F8EDEB}    90.38               & \cellcolor[HTML]{F8EDEB}    72.43                & \cellcolor[HTML]{F8EDEB}    75.17           & \cellcolor[HTML]{F8EDEB}    57.68           & \cellcolor[HTML]{F8EDEB}    80.35  & \cellcolor[HTML]{F8EDEB}    76.15        \\ 
                    & \cellcolor[HTML]{F8EDEB}                       & \cellcolor[HTML]{F8EDEB}    Generator           & \cellcolor[HTML]{F8EDEB}    71.36           & \cellcolor[HTML]{F8EDEB}    83.73          & \cellcolor[HTML]{F8EDEB}    79.54          & \cellcolor[HTML]{F8EDEB}    91.27              & \cellcolor[HTML]{F8EDEB}    83.62                & \cellcolor[HTML]{F8EDEB}    83.61           & \cellcolor[HTML]{F8EDEB}    66.32           & \cellcolor[HTML]{F8EDEB}    81.54   & \cellcolor[HTML]{F8EDEB}     80.12        \\ 
                    & \cellcolor[HTML]{F3DBD7}              & \cellcolor[HTML]{F3DBD7}    PredGen            & \cellcolor[HTML]{F3DBD7}    74.18           & \cellcolor[HTML]{F3DBD7}    85.88          & \cellcolor[HTML]{F3DBD7}    81.41          & \cellcolor[HTML]{F3DBD7}    93.62               & \cellcolor[HTML]{F3DBD7}    \textbf{86.76}                & \cellcolor[HTML]{F3DBD7}    86.25           & \cellcolor[HTML]{F3DBD7}    69.58           & \cellcolor[HTML]{F3DBD7}    84.77   & \cellcolor[HTML]{F3DBD7}    82.81       \\ 
\hline
\multirow{12}{*}{\textbf{Llama2-8B}} 
                    & \cellcolor[HTML]{ECEEFF}    LoRA              & \cellcolor[HTML]{ECEEFF}    Predictor           & \cellcolor[HTML]{ECEEFF}    68.44           & \cellcolor[HTML]{ECEEFF}    82.93          & \cellcolor[HTML]{ECEEFF}    79.84          & \cellcolor[HTML]{ECEEFF}    91.47               & \cellcolor[HTML]{ECEEFF}    71.58                & \cellcolor[HTML]{ECEEFF}    77.97           & \cellcolor[HTML]{ECEEFF}    56.02           & \cellcolor[HTML]{ECEEFF}    74.49  & \cellcolor[HTML]{ECEEFF}     75.34        \\ 
                    & \cellcolor[HTML]{ECEEFF}                     & \cellcolor[HTML]{ECEEFF}    Generator           & \cellcolor[HTML]{ECEEFF}    71.31           & \cellcolor[HTML]{ECEEFF}    81.45          & \cellcolor[HTML]{ECEEFF}    79.05          & \cellcolor[HTML]{ECEEFF}    90.65               & \cellcolor[HTML]{ECEEFF}    82.46                & \cellcolor[HTML]{ECEEFF}    82.83           & \cellcolor[HTML]{ECEEFF}    62.33           & \cellcolor[HTML]{ECEEFF}    76.64   & \cellcolor[HTML]{ECEEFF}    78.34       \\ 
                   & \cellcolor[HTML]{D8DBF5}                   & \cellcolor[HTML]{D8DBF5}    PredGen            & \cellcolor[HTML]{D8DBF5}    72.57           & \cellcolor[HTML]{D8DBF5}    83.63          & \cellcolor[HTML]{D8DBF5}    \textbf{81.72}          & \cellcolor[HTML]{D8DBF5}    92.98               & \cellcolor[HTML]{D8DBF5}    84.76                & \cellcolor[HTML]{D8DBF5}    \textbf{84.78}           & \cellcolor[HTML]{D8DBF5}    \textbf{64.64}           & \cellcolor[HTML]{D8DBF5}    80.54  & \cellcolor[HTML]{D8DBF5}    80.70        \\ 
                    & \cellcolor[HTML]{ECEEFF}     AdaLoRA           & \cellcolor[HTML]{ECEEFF}    Predictor           & \cellcolor[HTML]{ECEEFF}    68.11           & \cellcolor[HTML]{ECEEFF}    81.50           & \cellcolor[HTML]{ECEEFF}    79.88           & \cellcolor[HTML]{ECEEFF}    89.49               & \cellcolor[HTML]{ECEEFF}    71.37                & \cellcolor[HTML]{ECEEFF}    78.97           & \cellcolor[HTML]{ECEEFF}    54.72    & \cellcolor[HTML]{ECEEFF}    75.63   & \cellcolor[HTML]{ECEEFF}    74.96        \\ 
                    & \cellcolor[HTML]{ECEEFF}                   & \cellcolor[HTML]{ECEEFF}    Generator           & \cellcolor[HTML]{ECEEFF}    70.62           & \cellcolor[HTML]{ECEEFF}    82.48  & \cellcolor[HTML]{ECEEFF}    79.15          & \cellcolor[HTML]{ECEEFF}    91.17               & \cellcolor[HTML]{ECEEFF}    83.13                & \cellcolor[HTML]{ECEEFF}    82.62           & \cellcolor[HTML]{ECEEFF}    61.77           & \cellcolor[HTML]{ECEEFF}    78.53        & \cellcolor[HTML]{ECEEFF}    78.68  \\ 
                    & \cellcolor[HTML]{D8DBF5}                & \cellcolor[HTML]{D8DBF5}    PredGen            & \cellcolor[HTML]{D8DBF5}    73.10           & \cellcolor[HTML]{D8DBF5}    \textbf{84.88}          & \cellcolor[HTML]{D8DBF5}    80.61          & \cellcolor[HTML]{D8DBF5}    \textbf{93.22}               & \cellcolor[HTML]{D8DBF5}    85.23                & \cellcolor[HTML]{D8DBF5}    84.72           & \cellcolor[HTML]{D8DBF5}    62.81           & \cellcolor[HTML]{D8DBF5}    81.56  & \cellcolor[HTML]{D8DBF5}    80.77        \\ 
                    & \cellcolor[HTML]{ECEEFF}    RoCoFT            & \cellcolor[HTML]{ECEEFF}    Predictor           & \cellcolor[HTML]{ECEEFF}    67.96           & \cellcolor[HTML]{ECEEFF}    76.59          & \cellcolor[HTML]{ECEEFF}    79.92          & \cellcolor[HTML]{ECEEFF}    89.63               & \cellcolor[HTML]{ECEEFF}    72.02                & \cellcolor[HTML]{ECEEFF}    76.39           & \cellcolor[HTML]{ECEEFF}    54.92           & \cellcolor[HTML]{ECEEFF}    74.41  & \cellcolor[HTML]{ECEEFF}    73.98        \\ 
                    & \cellcolor[HTML]{ECEEFF}                   & \cellcolor[HTML]{ECEEFF}    Generator           & \cellcolor[HTML]{ECEEFF}    71.79           & \cellcolor[HTML]{ECEEFF}    83.23          & \cellcolor[HTML]{ECEEFF}    79.37          & \cellcolor[HTML]{ECEEFF}    90.84               & \cellcolor[HTML]{ECEEFF}    82.74                & \cellcolor[HTML]{ECEEFF}    82.67           & \cellcolor[HTML]{ECEEFF}    62.03           & \cellcolor[HTML]{ECEEFF}    77.73  & \cellcolor[HTML]{ECEEFF}    78.80 \\ 
                    & \cellcolor[HTML]{D8DBF5}        & \cellcolor[HTML]{D8DBF5}    PredGen            & \cellcolor[HTML]{D8DBF5}    \textbf{74.76}           & \cellcolor[HTML]{D8DBF5}    84.81          & \cellcolor[HTML]{D8DBF5}    80.86    & \cellcolor[HTML]{D8DBF5}    92.44               & \cellcolor[HTML]{D8DBF5}    \textbf{85.87}                & \cellcolor[HTML]{D8DBF5}    84.49    & \cellcolor[HTML]{D8DBF5}    62.97  & \cellcolor[HTML]{D8DBF5}    81.14     & \cellcolor[HTML]{D8DBF5}     \textbf{80.92}    \\ 
                    & \cellcolor[HTML]{ECEEFF}     DoRA              & \cellcolor[HTML]{ECEEFF}    Predictor           & \cellcolor[HTML]{ECEEFF}    67.88           & \cellcolor[HTML]{ECEEFF}    82.12          & \cellcolor[HTML]{ECEEFF}    80.26          & \cellcolor[HTML]{ECEEFF}    91.68     & \cellcolor[HTML]{ECEEFF}    71.68                & \cellcolor[HTML]{ECEEFF}    76.36           & \cellcolor[HTML]{ECEEFF}    54.42           & \cellcolor[HTML]{ECEEFF}    77.57  & \cellcolor[HTML]{ECEEFF}    75.25\\ 
                    & \cellcolor[HTML]{ECEEFF}                  & \cellcolor[HTML]{ECEEFF}    Generator           & \cellcolor[HTML]{ECEEFF}    72.32           & \cellcolor[HTML]{ECEEFF}    82.38          & \cellcolor[HTML]{ECEEFF}    80.01          & \cellcolor[HTML]{ECEEFF}    90.85               & \cellcolor[HTML]{ECEEFF}    83.38                & \cellcolor[HTML]{ECEEFF}    82.48           & \cellcolor[HTML]{ECEEFF}    61.03   & \cellcolor[HTML]{ECEEFF}    78.40      & \cellcolor[HTML]{ECEEFF}     78.86    \\ 
                    & \cellcolor[HTML]{D8DBF5}        & \cellcolor[HTML]{D8DBF5}    PredGen            & \cellcolor[HTML]{D8DBF5}    74.21           & \cellcolor[HTML]{D8DBF5}    83.59          & \cellcolor[HTML]{D8DBF5}    81.24          & \cellcolor[HTML]{D8DBF5}    93.17    & \cellcolor[HTML]{D8DBF5}    84.99                & \cellcolor[HTML]{D8DBF5}    84.72   & \cellcolor[HTML]{D8DBF5}    62.26           & \cellcolor[HTML]{D8DBF5}    \textbf{81.68}      & \cellcolor[HTML]{D8DBF5}     80.73   \\ \hline
\end{tabular}}
\caption{Performance of Classification with Different PEFT Methods Across Benchmarks. The best results are highlighted in bold for each model.}
\label{tab:classification}
\end{table*}

\paragraph{Main Results:} Table \ref{tab:classification} presents the classification performance of Llama models using different PEFT methods. PredGen consistently outperforms both the Predictor and Generator models across all tasks. For Llama2-7B, PredGen achieves an average accuracy of 79.67\%, surpassing both the Predictor (73.49\%) and Generator (76.63\%). Similarly, for Llama2-13B, PredGen reaches an average accuracy of 82.71\%, outperforming the other methods (76.20\% for Predictor and 80.40\% for Generator). Finally, for Llama2-8B, PredGen achieves an average accuracy of 80.92\%, again showing superior performance compared to the other models.

\begin{table*}[ht]
\centering
\resizebox{\textwidth}{!}{%
\begin{tabular}{l|l|l|c|c|c|c|c|c|c}
% \\ \hline
\rowcolor{gray!20}
\textbf{Model}     & \textbf{PEFT}    & \textbf{Method}    & \textbf{ WASSA } & \textbf{SICK} & \textbf{STSB} & \textbf{LCP} & \textbf{CRP} & \textbf{Humicroedit} & \textbf{Avg.}\\ \hline
\multirow{12}{*}{\textbf{Llama2-7B}} 
                    & \cellcolor[HTML]{EAF3FA}    LoRA              & \cellcolor[HTML]{EAF3FA}    Predictor           & \cellcolor[HTML]{EAF3FA}    0.454/0.151         & \cellcolor[HTML]{EAF3FA}    0.860/0.280         & \cellcolor[HTML]{EAF3FA}    0.965/0.950         & \cellcolor[HTML]{EAF3FA}    0.930/0.105              & \cellcolor[HTML]{EAF3FA}    1.014/0.784    & \cellcolor[HTML]{EAF3FA}    1.348/1.046  & \cellcolor[HTML]{EAF3FA}    0.928/0.553\\ 
                    & \cellcolor[HTML]{EAF3FA}                     & \cellcolor[HTML]{EAF3FA}    Generator           & \cellcolor[HTML]{EAF3FA}    0.090/0.023           & \cellcolor[HTML]{EAF3FA}    0.340/0.195         & \cellcolor[HTML]{EAF3FA}    0.610/0.630          & \cellcolor[HTML]{EAF3FA}    0.900/0.105              & \cellcolor[HTML]{EAF3FA}    0.465/0.349  & \cellcolor[HTML]{EAF3FA}    0.650/0.505  & \cellcolor[HTML]{EAF3FA}    0.509/0.301 \\ 
                    & \cellcolor[HTML]{D0E7F7}               & \cellcolor[HTML]{D0E7F7}    PredGen           & \cellcolor[HTML]{D0E7F7}    0.088/0.022           & \cellcolor[HTML]{D0E7F7}    0.320/0.190          & \cellcolor[HTML]{D0E7F7}    0.576/\textbf{0.569}         & \cellcolor[HTML]{D0E7F7}    0.062/0.008               & \cellcolor[HTML]{D0E7F7}    0.420/0.280  & \cellcolor[HTML]{D0E7F7}    0.550/0.455  & \cellcolor[HTML]{D0E7F7}    0.338/0.257\\ 
                    & \cellcolor[HTML]{EAF3FA}     AdaLoRA           & \cellcolor[HTML]{EAF3FA}    Predictor           & \cellcolor[HTML]{EAF3FA}    0.424/0.148         & \cellcolor[HTML]{EAF3FA}    0.845/0.270          & \cellcolor[HTML]{EAF3FA}    0.950/0.935        & \cellcolor[HTML]{EAF3FA}    0.918/0.100             & \cellcolor[HTML]{EAF3FA}    1.020/0.790  & \cellcolor[HTML]{EAF3FA}    1.360/1.050  & \cellcolor[HTML]{EAF3FA}    0.920/0.549 \\ 
                    & \cellcolor[HTML]{EAF3FA}            & \cellcolor[HTML]{EAF3FA}    Generator           & \cellcolor[HTML]{EAF3FA}    0.087/0.022         & \cellcolor[HTML]{EAF3FA}    0.325/0.185         & \cellcolor[HTML]{EAF3FA}    0.600/0.620         & \cellcolor[HTML]{EAF3FA}    0.890/0.097              & \cellcolor[HTML]{EAF3FA}    0.455/0.335     & \cellcolor[HTML]{EAF3FA}     0.630/0.490   & \cellcolor[HTML]{EAF3FA}    0.498/0.291 \\ 
                    & \cellcolor[HTML]{D0E7F7}        & \cellcolor[HTML]{D0E7F7}    PredGen            & \cellcolor[HTML]{D0E7F7}    \textbf{0.080/0.020}           & \cellcolor[HTML]{D0E7F7}    0.305/0.185        & \cellcolor[HTML]{D0E7F7}    \textbf{0.575}/0.570         & \cellcolor[HTML]{D0E7F7}    \textbf{0.058/0.006}               & \cellcolor[HTML]{D0E7F7}    \textbf{0.405/0.270}        & \cellcolor[HTML]{D0E7F7}     \textbf{0.535/0.440}  & \cellcolor[HTML]{D0E7F7}    0.326/0.248  \\ 
                    & \cellcolor[HTML]{EAF3FA}    RoCoFT           & \cellcolor[HTML]{EAF3FA}    Predictor           & \cellcolor[HTML]{EAF3FA}    0.424/0.148           & \cellcolor[HTML]{EAF3FA}    0.854/0.274         & \cellcolor[HTML]{EAF3FA}    0.958/0.942         & \cellcolor[HTML]{EAF3FA}    0.924/0.102              & \cellcolor[HTML]{EAF3FA}    0.990/0.770   & \cellcolor[HTML]{EAF3FA}    1.340/1.040  & \cellcolor[HTML]{EAF3FA}   0.915/0.546\\ 
                    & \cellcolor[HTML]{EAF3FA}                    & \cellcolor[HTML]{EAF3FA}    Generator           & \cellcolor[HTML]{EAF3FA}    0.085/0.021          & \cellcolor[HTML]{EAF3FA}    0.332/0.191          & \cellcolor[HTML]{EAF3FA}    0.605/0.623         & \cellcolor[HTML]{EAF3FA}    0.895/0.099              & \cellcolor[HTML]{EAF3FA}    0.460/0.337    & \cellcolor[HTML]{EAF3FA}     0.641/0.497     & \cellcolor[HTML]{EAF3FA}     0.503/0.295  \\ 
                    & \cellcolor[HTML]{D0E7F7}                    & \cellcolor[HTML]{D0E7F7}    PredGen            & \cellcolor[HTML]{D0E7F7}    0.084/0.021          & \cellcolor[HTML]{D0E7F7}    0.311/0.187         & \cellcolor[HTML]{D0E7F7}    0.583/0.580         & \cellcolor[HTML]{D0E7F7}    0.06/0.007               & \cellcolor[HTML]{D0E7F7}    \textbf{0.405}/0.274   & \cellcolor[HTML]{D0E7F7}    0.543/0.448  & \cellcolor[HTML]{D0E7F7}    0.332/0.253 \\ 
                    & \cellcolor[HTML]{EAF3FA}  DoRA              & \cellcolor[HTML]{EAF3FA}    Predictor           & \cellcolor[HTML]{EAF3FA}    0.511/0.150           & \cellcolor[HTML]{EAF3FA}    0.850/0.275  & \cellcolor[HTML]{EAF3FA}    0.960/0.945         & \cellcolor[HTML]{EAF3FA}     0.922/0.104               & \cellcolor[HTML]{EAF3FA}    0.980/0.780       & \cellcolor[HTML]{EAF3FA}     1.355/1.048  & \cellcolor[HTML]{EAF3FA}    0.930/0.550\\ 
                    & \cellcolor[HTML]{EAF3FA}                   & \cellcolor[HTML]{EAF3FA}    Generator           & \cellcolor[HTML]{EAF3FA}    0.086/0.022          & \cellcolor[HTML]{EAF3FA}    0.330/0.190          & \cellcolor[HTML]{EAF3FA}    0.607/0.625          & \cellcolor[HTML]{EAF3FA}    0.885/0.100               & \cellcolor[HTML]{EAF3FA}    0.462/0.338      & \cellcolor[HTML]{EAF3FA}      0.645/0.500  & \cellcolor[HTML]{EAF3FA}     0.503/0.296 \\ 
                    & \cellcolor[HTML]{D0E7F7}     & \cellcolor[HTML]{D0E7F7}    PredGen            & \cellcolor[HTML]{D0E7F7}    0.085/0.021          & \cellcolor[HTML]{D0E7F7}    \textbf{0.301/0.184}          & \cellcolor[HTML]{D0E7F7}    0.580/0.578        & \cellcolor[HTML]{D0E7F7}    0.061/0.007               & \cellcolor[HTML]{D0E7F7}    0.415/0.275    & \cellcolor[HTML]{D0E7F7}     0.540/0.445    & \cellcolor[HTML]{D0E7F7}    0.333/0.252 \\ 
\hline
\multirow{12}{*}{\textbf{Llama2-13B}} 
                    & \cellcolor[HTML]{F8EDEB}    LoRA         & \cellcolor[HTML]{F8EDEB}    Predictor           & \cellcolor[HTML]{F8EDEB}    0.370/0.130         & \cellcolor[HTML]{F8EDEB}    0.800/0.250          & \cellcolor[HTML]{F8EDEB}    0.920/0.910    & \cellcolor[HTML]{F8EDEB}    0.880/0.090             & \cellcolor[HTML]{F8EDEB}    0.950/0.720   & \cellcolor[HTML]{F8EDEB}     1.280/1.000   & \cellcolor[HTML]{F8EDEB}    0.867/0.517 \\ 
                    & \cellcolor[HTML]{F8EDEB}            & \cellcolor[HTML]{F8EDEB}    Generator           & \cellcolor[HTML]{F8EDEB}    0.075/0.018          & \cellcolor[HTML]{F8EDEB}    0.310/0.175          & \cellcolor[HTML]{F8EDEB}    0.580/0.590          & \cellcolor[HTML]{F8EDEB}    0.850/0.090               & \cellcolor[HTML]{F8EDEB}    0.430/0.310     & \cellcolor[HTML]{F8EDEB}     0.600/0.460      & \cellcolor[HTML]{F8EDEB}    0.474/0.274    \\ 
                    & \cellcolor[HTML]{F3DBD7}          & \cellcolor[HTML]{F3DBD7}    PredGen            & \cellcolor[HTML]{F3DBD7}    0.074/0.018          & \cellcolor[HTML]{F3DBD7}    \textbf{0.287/0.169}    & \cellcolor[HTML]{F3DBD7}    0.550/0.540         & \cellcolor[HTML]{F3DBD7}    0.052/0.006               & \cellcolor[HTML]{F3DBD7}    0.380/0.250     & \cellcolor[HTML]{F3DBD7}     0.500/\textbf{0.400}     & \cellcolor[HTML]{F3DBD7}    0.308/0.231   \\ 
                    & \cellcolor[HTML]{F8EDEB}     AdaLoRA           & \cellcolor[HTML]{F8EDEB}    Predictor           & \cellcolor[HTML]{F8EDEB}     0.360/0.125          & \cellcolor[HTML]{F8EDEB}    0.810/0.255   & \cellcolor[HTML]{F8EDEB}    0.930/0.920         & \cellcolor[HTML]{F8EDEB}    0.890/0.095   & \cellcolor[HTML]{F8EDEB}    0.960/0.730   & \cellcolor[HTML]{F8EDEB}     1.300/1.010  & \cellcolor[HTML]{F8EDEB}    0.875/0.522 \\ 
                    & \cellcolor[HTML]{F8EDEB}                  & \cellcolor[HTML]{F8EDEB}    Generator           & \cellcolor[HTML]{F8EDEB}    0.078/0.019         & \cellcolor[HTML]{F8EDEB}    0.315/0.178          & \cellcolor[HTML]{F8EDEB}    0.585/0.600          & \cellcolor[HTML]{F8EDEB}    0.860/0.093               & \cellcolor[HTML]{F8EDEB}    0.440/0.320     & \cellcolor[HTML]{F8EDEB}    0.610/0.470  & \cellcolor[HTML]{F8EDEB}     0.481/0.280 \\ 
                    & \cellcolor[HTML]{F3DBD7}                   & \cellcolor[HTML]{F3DBD7}    PredGen            & \cellcolor[HTML]{F3DBD7}    0.078/0.019          & \cellcolor[HTML]{F3DBD7}    0.300/0.175          & \cellcolor[HTML]{F3DBD7}    \textbf{0.530/0.530}         & \cellcolor[HTML]{F3DBD7}   0.054/0.006             & \cellcolor[HTML]{F3DBD7}    0.390/0.255  & \cellcolor[HTML]{F3DBD7}    0.510/0.410  & \cellcolor[HTML]{F3DBD7}    0.315/0.236 \\ 
                    & \cellcolor[HTML]{F8EDEB}    RoCoFT            & \cellcolor[HTML]{F8EDEB}    Predictor           & \cellcolor[HTML]{F8EDEB}    0.380/0.135           & \cellcolor[HTML]{F8EDEB}    0.790/0.245           & \cellcolor[HTML]{F8EDEB}    0.910/0.900           & \cellcolor[HTML]{F8EDEB}    0.870/0.088               & \cellcolor[HTML]{F8EDEB}    0.940/0.710  & \cellcolor[HTML]{F8EDEB}    1.270/0.990  & \cellcolor[HTML]{F8EDEB}    0.860/0.511 \\ 
                    & \cellcolor[HTML]{F8EDEB}                    & \cellcolor[HTML]{F8EDEB}    Generator           & \cellcolor[HTML]{F8EDEB}    0.072/0.017          & \cellcolor[HTML]{F8EDEB}    0.305/0.172          & \cellcolor[HTML]{F8EDEB}    0.575/0.580          & \cellcolor[HTML]{F8EDEB}    0.845/0.088              & \cellcolor[HTML]{F8EDEB}    0.425/0.305    & \cellcolor[HTML]{F8EDEB}     0.590/0.450  & \cellcolor[HTML]{F8EDEB}    0.860/0.511\\ 
                    & \cellcolor[HTML]{F3DBD7}                   & \cellcolor[HTML]{F3DBD7}    PredGen            & \cellcolor[HTML]{F3DBD7}    \textbf{0.070}/0.017         & \cellcolor[HTML]{F3DBD7}    0.288/\textbf{0.169}         & \cellcolor[HTML]{F3DBD7}    0.545/0.538        & \cellcolor[HTML]{F3DBD7}    \textbf{0.053/0.007}               & \cellcolor[HTML]{F3DBD7}    \textbf{0.375/0.248}     & \cellcolor[HTML]{F3DBD7}     \textbf{0.495}/0.401  & \cellcolor[HTML]{F3DBD7}    0.307/0.232\\ 
                    & \cellcolor[HTML]{F8EDEB}     DoRA              & \cellcolor[HTML]{F8EDEB}    Predictor           & \cellcolor[HTML]{F8EDEB}    0.365/0.128           & \cellcolor[HTML]{F8EDEB}    0.805/0.252          & \cellcolor[HTML]{F8EDEB}    0.925/0.915         & \cellcolor[HTML]{F8EDEB}    0.924/0.102              & \cellcolor[HTML]{F8EDEB}    0.955/0.725   & \cellcolor[HTML]{F8EDEB}    1.290/1.005  & \cellcolor[HTML]{F8EDEB}    0.877/0.521 \\ 
                    & \cellcolor[HTML]{F8EDEB}                   & \cellcolor[HTML]{F8EDEB}    Generator           & \cellcolor[HTML]{F8EDEB}    0.076/0.018          & \cellcolor[HTML]{F8EDEB}    0.312/0.176        & \cellcolor[HTML]{F8EDEB}    0.590/0.605         & \cellcolor[HTML]{F8EDEB}    0.855/0.092             & \cellcolor[HTML]{F8EDEB}    0.435/0.315    & \cellcolor[HTML]{F8EDEB}     0.605/0.465  & \cellcolor[HTML]{F8EDEB}    0.479/0.279\\ 
                    & \cellcolor[HTML]{F3DBD7}            & \cellcolor[HTML]{F3DBD7}    PredGen            & \cellcolor[HTML]{F3DBD7}    \textbf{0.070/0.016}           & \cellcolor[HTML]{F3DBD7}    0.295/0.172        & \cellcolor[HTML]{F3DBD7}    0.555/0.548          & \cellcolor[HTML]{F3DBD7}    0.053/0.006               & \cellcolor[HTML]{F3DBD7}    0.385/0.252   & \cellcolor[HTML]{F3DBD7}     0.505/0.405   & \cellcolor[HTML]{F3DBD7}    0.311/0.233 \\ 
\hline
\multirow{12}{*}{\textbf{Llama2-8B}} 
                    & \cellcolor[HTML]{ECEEFF}    LoRA              & \cellcolor[HTML]{ECEEFF}    Predictor           & \cellcolor[HTML]{ECEEFF}    0.380/0.140           & \cellcolor[HTML]{ECEEFF}    0.820/0.260          & \cellcolor[HTML]{ECEEFF}    0.940/0.925          & \cellcolor[HTML]{ECEEFF}    0.910/0.098              & \cellcolor[HTML]{ECEEFF}    0.970/0.740  & \cellcolor[HTML]{ECEEFF}    1.310/1.020   & \cellcolor[HTML]{ECEEFF}    0.888/0.531\\ 
                    & \cellcolor[HTML]{ECEEFF}                    & \cellcolor[HTML]{ECEEFF}    Generator           & \cellcolor[HTML]{ECEEFF}    0.081/0.019         & \cellcolor[HTML]{ECEEFF}    0.320/0.180          & \cellcolor[HTML]{ECEEFF}    0.595/0.610          & \cellcolor[HTML]{ECEEFF}    0.870/0.095            & \cellcolor[HTML]{ECEEFF}    0.440/0.325    & \cellcolor[HTML]{ECEEFF}    0.620/0.480   & \cellcolor[HTML]{ECEEFF}     0.488/0.285 \\ 
                    & \cellcolor[HTML]{D8DBF5}                   & \cellcolor[HTML]{D8DBF5}    PredGen            & \cellcolor[HTML]{D8DBF5}    0.077/0.019           & \cellcolor[HTML]{D8DBF5}    0.298/0.173          & \cellcolor[HTML]{D8DBF5}    0.565/0.555         & \cellcolor[HTML]{D8DBF5}    \textbf{0.055/0.006}              & \cellcolor[HTML]{D8DBF5}    0.395/0.260      & \cellcolor[HTML]{D8DBF5}    0.520/0.420  & \cellcolor[HTML]{D8DBF5}   0.318/0.239 \\ 
                    & \cellcolor[HTML]{ECEEFF}    AdaLoRA           & \cellcolor[HTML]{ECEEFF}    Predictor           & \cellcolor[HTML]{ECEEFF}    0.375/0.135           & \cellcolor[HTML]{ECEEFF}    0.830/0.265          & \cellcolor[HTML]{ECEEFF}    0.945/0.930         & \cellcolor[HTML]{ECEEFF}    0.910/0.098               & \cellcolor[HTML]{ECEEFF}    0.980/0.750    & \cellcolor[HTML]{ECEEFF}      1.320/1.030  & \cellcolor[HTML]{ECEEFF}    0.893/0.535\\ 
                    & \cellcolor[HTML]{ECEEFF}                    & \cellcolor[HTML]{ECEEFF}    Generator           & \cellcolor[HTML]{ECEEFF}    0.080/0.020           & \cellcolor[HTML]{ECEEFF}    0.325/0.183  & \cellcolor[HTML]{ECEEFF}    0.600/0.615         & \cellcolor[HTML]{ECEEFF}    0.875/0.097              & \cellcolor[HTML]{ECEEFF}    0.450/0.330     & \cellcolor[HTML]{ECEEFF}     0.630/0.485   & \cellcolor[HTML]{ECEEFF}    0.493/0.288 \\ 
                    & \cellcolor[HTML]{D8DBF5}                   & \cellcolor[HTML]{D8DBF5}    PredGen            & \cellcolor[HTML]{D8DBF5}    0.078/0.019          & \cellcolor[HTML]{D8DBF5}    0.303/0.177          & \cellcolor[HTML]{D8DBF5}    0.570/0.560          & \cellcolor[HTML]{D8DBF5}    0.057/0.007              & \cellcolor[HTML]{D8DBF5}    0.400/0.265     & \cellcolor[HTML]{D8DBF5}    \textbf{0.509/0.410}  & \cellcolor[HTML]{D8DBF5}     0.323/0.243  \\ 
                    & \cellcolor[HTML]{ECEEFF}     RoCoFT            & \cellcolor[HTML]{ECEEFF}    Predictor           & \cellcolor[HTML]{ECEEFF}    0.390/0.145          & \cellcolor[HTML]{ECEEFF}    0.810/0.255          & \cellcolor[HTML]{ECEEFF}    0.935/0.920         & \cellcolor[HTML]{ECEEFF}    0.910/0.098              & \cellcolor[HTML]{ECEEFF}    0.960/0.730     & \cellcolor[HTML]{ECEEFF}    1.300/1.015   & \cellcolor[HTML]{ECEEFF}    0.884/0.527 \\ 
                    & \cellcolor[HTML]{ECEEFF}                & \cellcolor[HTML]{ECEEFF}    Generator           & \cellcolor[HTML]{ECEEFF}    0.082/0.020   & \cellcolor[HTML]{ECEEFF}    0.315/0.177         & \cellcolor[HTML]{ECEEFF}    0.585/0.605          & \cellcolor[HTML]{ECEEFF}    0.865/0.092              & \cellcolor[HTML]{ECEEFF}    0.435/0.320      & \cellcolor[HTML]{ECEEFF}     0.610/0.475   & \cellcolor[HTML]{ECEEFF}    0.482/0.282\\ 
                    & \cellcolor[HTML]{D8DBF5}                   & \cellcolor[HTML]{D8DBF5}    PredGen            & \cellcolor[HTML]{D8DBF5}    0.079/0.020           & \cellcolor[HTML]{D8DBF5}    \textbf{0.288/0.169}         & \cellcolor[HTML]{D8DBF5}    0.565/\textbf{0.558}    & \cellcolor[HTML]{D8DBF5}    0.058/0.007             & \cellcolor[HTML]{D8DBF5}     \textbf{0.385/0.255}  & \cellcolor[HTML]{D8DBF5}    0.530/0.425  & \cellcolor[HTML]{D8DBF5}    0.317/0.238 \\ 
                    & \cellcolor[HTML]{ECEEFF}    DoRA              & \cellcolor[HTML]{ECEEFF}    Predictor           & \cellcolor[HTML]{ECEEFF}    0.385/0.138           & \cellcolor[HTML]{ECEEFF}    0.825/0.261   & \cellcolor[HTML]{ECEEFF}    0.950/0.935         & \cellcolor[HTML]{ECEEFF}    0.905/0.096     & \cellcolor[HTML]{ECEEFF}    0.975/0.745  & \cellcolor[HTML]{ECEEFF}    1.315/1.025   & \cellcolor[HTML]{ECEEFF}     0.893/0.533     \\ 
                    & \cellcolor[HTML]{ECEEFF}                & \cellcolor[HTML]{ECEEFF}    Generator           & \cellcolor[HTML]{ECEEFF}    0.078/0.019          & \cellcolor[HTML]{ECEEFF}    0.322/0.179        & \cellcolor[HTML]{ECEEFF}    0.592/0.608         & \cellcolor[HTML]{ECEEFF}    0.880/0.096               & \cellcolor[HTML]{ECEEFF}    0.445/0.328  & \cellcolor[HTML]{ECEEFF}    0.625/0.482  & \cellcolor[HTML]{ECEEFF}   0.490/0.285\\ 
                    & \cellcolor[HTML]{D8DBF5}                  & \cellcolor[HTML]{D8DBF5}    PredGen            & \cellcolor[HTML]{D8DBF5}    \textbf{0.073/0.018}          & \cellcolor[HTML]{D8DBF5}    0.300/0.175         & \cellcolor[HTML]{D8DBF5}    \textbf{0.562/0.558}         & \cellcolor[HTML]{D8DBF5}    0.066/0.007    & \cellcolor[HTML]{D8DBF5}    0.390/0.262   & \cellcolor[HTML]{D8DBF5}    0.525/0.425     & \cellcolor[HTML]{D8DBF5}     0.319/0.241  \\ 
\hline
\end{tabular}%
}
\caption{Performance of regression with Different PEFT Methods Across Benchmarks.}
\label{tab:regression}
\end{table*}

Table \ref{tab:regression} presents the regression performance where PredGen consistently outperforms both the Predictor and Generator models in most tasks. For Llama2-7B, PredGen achieves an average score of 0.338, outperforming the Predictor (0.928) and Generator (0.509). Similarly, for Llama2-13B, PredGen shows an average score of 0.308, better than the Predictor (0.867) and Generator (0.474). Finally, for Llama2-8B, PredGen reaches an average of 0.319, surpassing both the Predictor (0.888) and Generator (0.493). These results indicate that PredGen outperforms standard approaches in classification and regression tasks because generating predictions as token sequences carries more mutual information, leading to higher accuracy.

\section{Ablation Study}
\begin{figure}[!t]
    \centering
    \includegraphics[width=0.9980\linewidth]{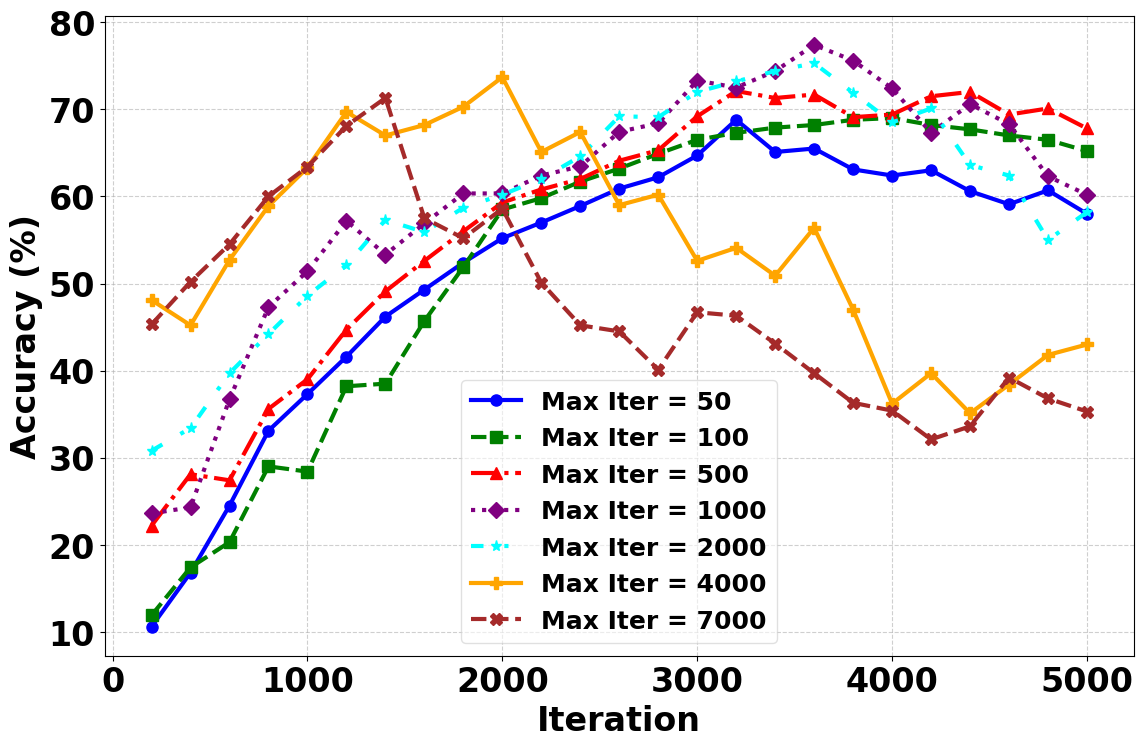}
  
    \caption{Effect of \texttt{max\_steps\_for\_sampling} on performance. A gradual transition (\texttt{max\_steps\_for\_sampling} = 1000) achieves the best performance, balancing reference-based and self-generated predictions.}

\label{fig:sampling}
\end{figure}
\paragraph{Scheduled Sampling vs. Performance}

We analyze how different scheduled sampling strategies affect accuracy on the MultiArith dataset \cite{roy2016solving}. Instead of always using ground-truth tokens as input, the model gradually shifts from using reference tokens to using its own generated tokens. Figure \ref{fig:sampling} shows how varying \texttt{max\_steps\_for\_sampling} impacts the performance.

If the transition happens too quickly (e.g., \texttt{max\_steps\_for\_sampling} = 50 or 100), the model receives too few reference tokens, leading to unstable predictions. On the other hand, if the transition is too slow (e.g., \texttt{max\_steps\_for\_sampling} = 7000), the model remains overly dependent on reference tokens and struggles during inference.

The best results occur when \texttt{max\_steps\_for\_sampling} = 1000, striking a balance between guidance from reference tokens and adaptation using self-generated tokens. This suggests that carefully tuning the transition period is key to improving the accuracy of generative prediction tasks.

\begin{figure}[!t]
    \centering
    \includegraphics[width=1.0\linewidth]{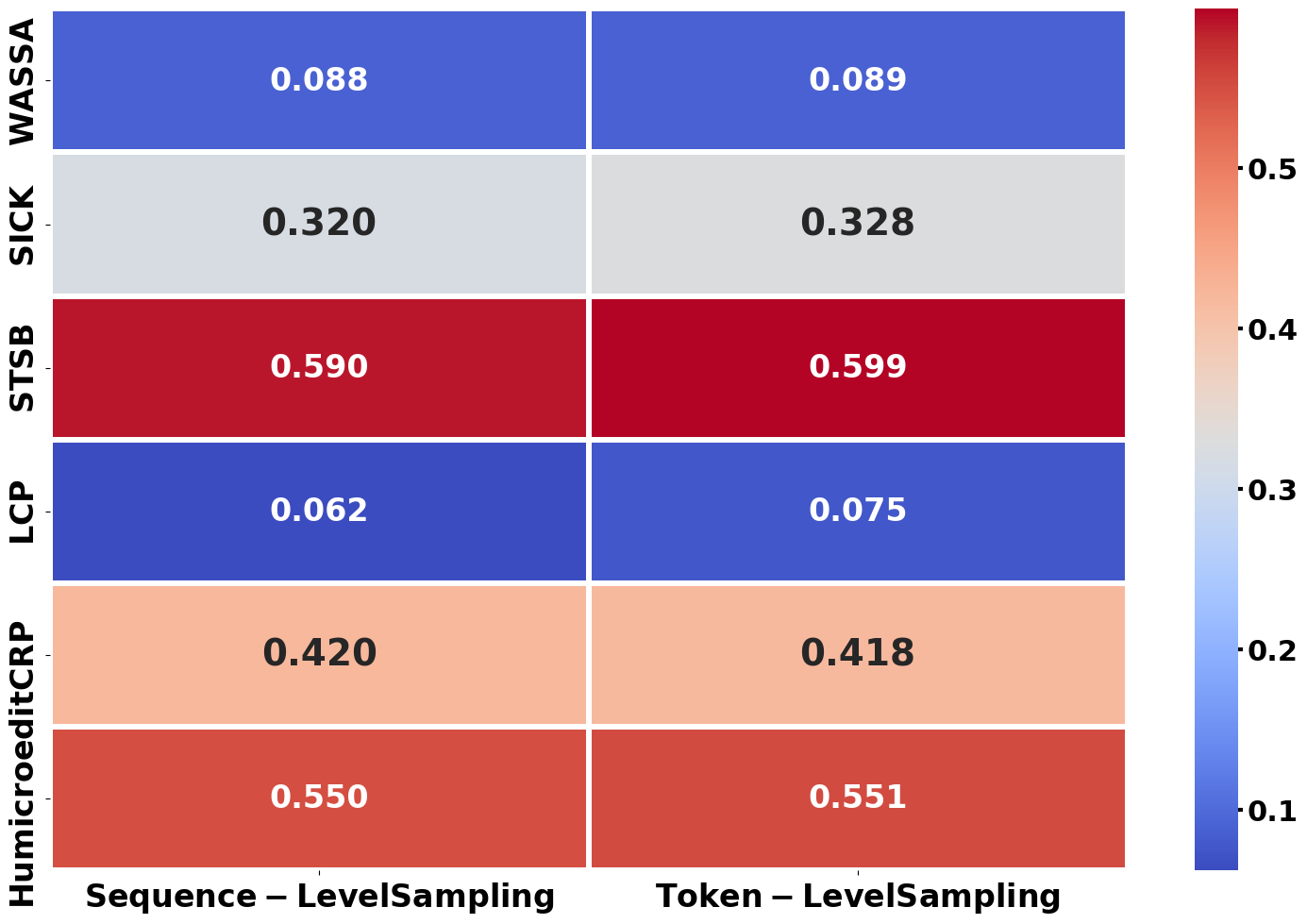}
\caption{MSE loss comparison between Sequence-Level and Token-Level Sampling across datasets.}
\label{fig:sampling_method_cmp}
\end{figure}
\paragraph{Token-Level vs. Sequence-Level Scheduled Sampling}

We analyze the impact of two different approaches in scheduled sampling: \textbf{Sequence-Level Sampling} and \textbf{Token-Level Sampling}. In \emph{Sequence-Level Sampling}, the entire output sequence \(\mathbf{Y}\) is either fully generated by the model or fully replaced with ground-truth tokens during training. In contrast, \emph{Token-Level Sampling} selectively decides for each token \(\mathbf{Y_t}\) whether to use the generated output or ground truth, allowing a more gradual transition.

Figure~\ref{fig:sampling_method_cmp} presents the MSE loss comparison between these two schedule sampling methods across six datasets. We observe that Token-Level Sampling consistently achieves slightly lower MSE across most datasets, indicating better alignment between generated tokens and the true outputs. For example, on \texttt{STSB}, the MSE decreases from 0.590 to 0.599, and on \texttt{LCP}, it improves from 0.062 to 0.075. This suggests that allowing the model to mix the generated and ground-truth tokens at a finer level helps it adapt more smoothly, reducing the sharp transitions between training and inference.

\begin{figure}[!t]
    \centering
    \includegraphics[width=0.9980\linewidth]{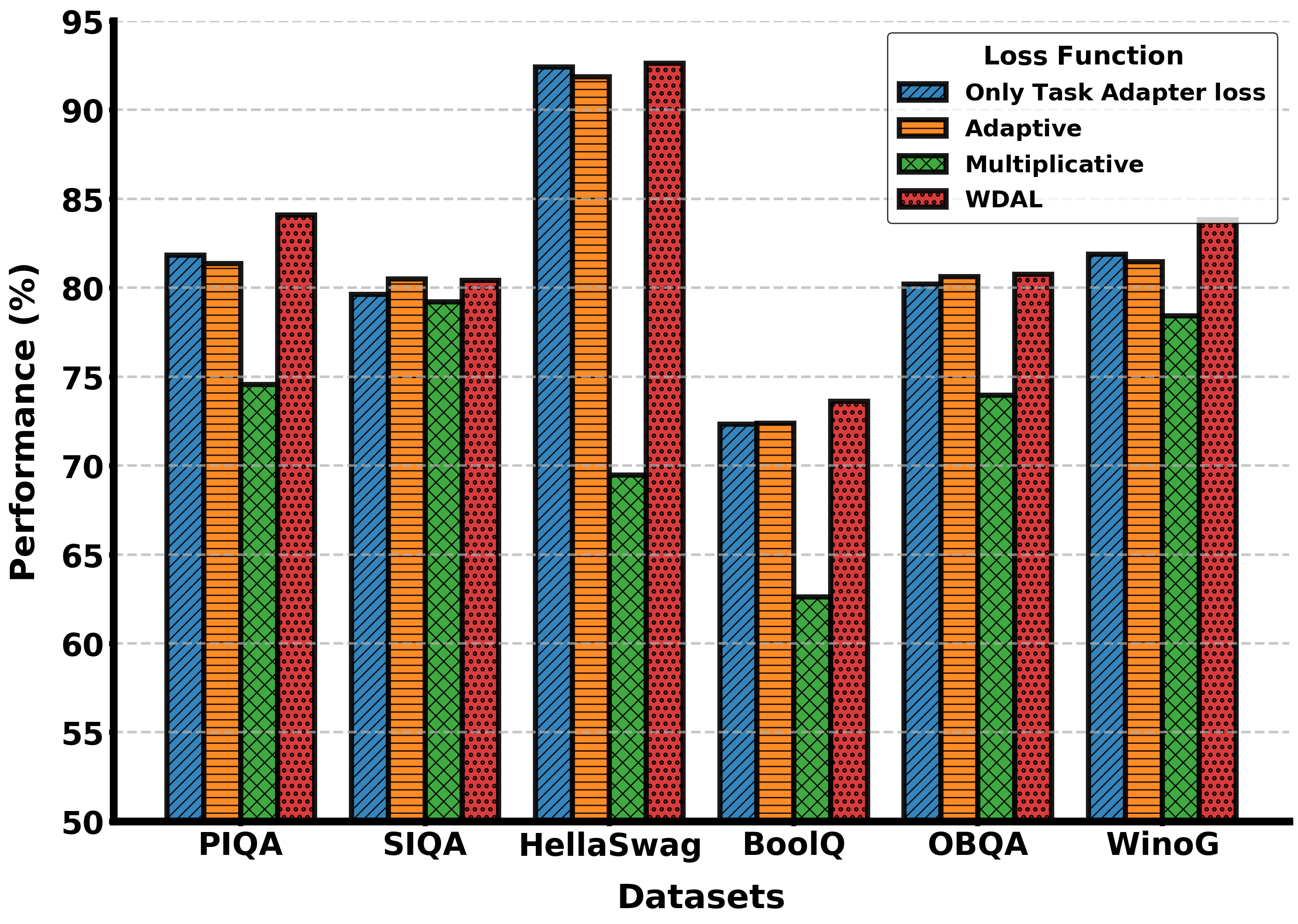}
\caption{Comparison of WDAL, Adaptive, and Multiplicative loss functions across multiple datasets.}

\label{fig:loss_cmp}
\end{figure}
\paragraph{Loss Function Comparison:}

We compare \textbf{WDAL} with three other loss functions: \textbf{Adaptive Loss}, \textbf{Multiplicative Loss}, and \textbf{Only Task Adapter Loss}. The \emph{Only Task Adapter Loss} trains the model using only the task adapter’s objective, without the generator’s loss. This means that the model focuses only on final predictions and does not align with the token-level generation process. Figure~\ref{fig:loss_cmp} shows the accuracy scores across the six datasets.

Our results show that \textbf{WDAL} consistently performs the best, achieving the highest accuracy across all datasets, including \texttt{HellaSwag} (92.64\%), \texttt{WinoGrande} (83.85\%), and \texttt{BoolQ} (73.62\%). The \emph{Only Task Adapter Loss} performs the worst, confirming that ignoring the generator’s loss weakens the overall performance. This suggests that learning from both token generation and final predictions is important. The \emph{Multiplicative Loss} is unstable, especially on \texttt{HellaSwag} (69.48\%) and \texttt{BoolQ} (62.63\%), likely due to an imbalance between the two objectives. The \emph{Adaptive Loss} is more stable but still lags behind WDAL, as it does not properly align token-level generation with task predictions.

\section{Related Work}

\paragraph{Token-Level Generation:}  
Generating output tokens using LLMs has been widely explored in tasks such as question answering \cite{rajpurkar2016squad, yang2019xlnet}, summarization \cite{lewis2019bart, zhang2020pegasus}, and machine translation \cite{vaswani2017attention, wu2016google}. However, using token-level generation for supervised learning tasks like structured prediction remains underexplored. Recent studies \cite{chen2023token, yu2024breaking} show that token-level generation can be more effective than pooled representations by aligning with the LLMs' pre-training objective, leading to better efficiency and robustness against errors.

\paragraph{Mutual Information:}  
Mutual information (MI) helps measure dependencies between features in deep learning \cite{cover1999elements, covert2023learning}. In language models, \citet{chen2024learning} used MI for Chain-of-Thought Distillation, while the MIST framework \cite{kamthawee2024mist} applied it to short-text clustering. Unlike these works, we use MI to show that token-level generation retains more information than pooled representations.

\paragraph{Mitigating Exposure Bias and Format Mismatch:}  
Exposure bias occurs when an autoregressive model is trained with ground-truth tokens but must rely on its own predictions during inference. Scheduled sampling \cite{bengio2015scheduled} helps reduce this gap. Format mismatch arises when generated tokens do not align with the required structured output. \citet{wang2022code4struct} improved coherence by extracting structured information from LLMs, while \citet{liu2022autoregressive} modeled structured outputs as action sequences to preserve dependencies.

Unlike previous works, we extend token-level generation to both regression and classification and provide theoretical and empirical proof of its advantages. We also integrate scheduled sampling with a task adapter to ensure generated tokens match numerical or categorical outputs, addressing exposure bias and format mismatch.

\section{Conclusion}

In this work, we explored the advantages of generation-based prediction over traditional classifier-based approaches. We provided both theoretical and empirical evidence showing that token-level generation retains more task-relevant information than pooled representations. Using the \textit{DPI}, we proved that generating tokens preserves strictly more mutual information with the target output, addressing the limitations of classification-based fine-tuning. To tackle key challenges in generative prediction, we introduced \textbf{PredGen}, an end-to-end framework that mitigates \textit{exposure bias} using scheduled sampling and ensures structured outputs through a task adapter. Furthermore, we proposed the \textbf{WDAL}, which aligns token generation with task-specific objectives, leading to more coherent and numerically accurate predictions. Extensive experiments on classification and regression benchmarks demonstrated that PredGen consistently outperforms standard baselines.

\section{Limitations}

While \textsc{PredGen} demonstrates strong performance in structured prediction tasks, it has several limitations that warrant further investigation:

\begin{itemize}
    \item \textbf{Inference Latency:} Generation-based prediction introduces additional computational overhead due to the sequential nature of autoregressive decoding. Unlike classification-based methods that produce outputs in a single forward pass, \textsc{PredGen} generates outputs token by token, leading to increased inference time, especially for long outputs.
    
    \item \textbf{Exposure Bias:} Although scheduled sampling helps mitigate exposure bias, it does not fully eliminate the issue. During training, the model sees ground-truth tokens, but during inference, it must rely on its own generated outputs. This transition can still cause compounding errors, particularly in long-horizon predictions.
    
    \item \textbf{Increased Training Time:} Scheduled sampling increases training time since it requires generating tokens to obtain the final representation. As token generation is inherently sequential and non-parallelizable, this process slows down training compared to traditional classification-based fine-tuning.
    
    \item \textbf{Task-Specific Adaptation:} \textsc{PredGen} depends on carefully designed task adapters to map generated tokens into structured outputs. Developing effective adapters for different tasks may require task-specific tuning, limiting the framework’s adaptability to new domains without additional engineering effort.
\end{itemize}

Addressing these limitations in future work could further improve the efficiency and generalizability of \textsc{PredGen}.

\bibliography{acl_latex}

\clearpage

\appendix

\onecolumn
\tableofcontents
\twocolumn

\section{Writer-Director Alignment Loss}
\label{sec:wdal-appendix}

\subsection{Full Derivation of WDAL}
Here, we present the step-by-step derivation of the WDAL for completeness. The WDAL is defined 
\begin{multline*}
  L_{\text{WDAL}} = \max\!\bigl(L_W^2, L_D^2\bigr)\,\\\exp\!\Bigl(-\Bigl|\log L_W - \log L_D\Bigr|\Bigr).
\end{multline*}

\paragraph{Multiplicative Form:}
We begin by defining the overall WDAL as the product of the writer’s loss \(L_W\) and the director’s loss \(L_D\):
\[
L_{\text{WDAL}} = L_W \cdot L_D.
\]
The key rationale is twofold:
\begin{enumerate}
    \item If the writer fails (\(L_W\) is large), the director (conditioned on the writer) must also fail; hence the product remains large.
    \item If the writer succeeds (\(L_W\) is small) but the director fails (\(L_D\) is large), the product remains large, enforcing joint success.
\end{enumerate}

\paragraph{Log-Sum-Exp Trick:}
To ensure numerical stability, we consider the logarithms:
\[
\log L_{\text{WDAL}} \;=\; \log L_W \;+\; \log L_D.
\]
However, adding \(\log L_W\) and \(\log L_D\) directly can be numerically unstable if \(L_W\) or \(L_D\) has a large disparity. 

To further stabilize computations, we introduce 
\[
M \;=\; \log\bigl(\max(L_W, L_D)\bigr),
\]
and rewrite:
\small
\[
\log L_{\text{WDAL}} 
= (\log L_W - M) \;+\; (\log L_D - M) \;+\; 2M .
\]
\normalsize
Exponentiating both sides gives:
\[
L_{\text{WDAL}} 
\;=\; e^{(\log L_W - M)} 
\;\cdot\; e^{(\log L_D - M)}
\;\cdot\; e^{2M} .
\]
Substituting \(M = \log(\max(L_W, L_D))\), we can factor out the largest term squared:
\begin{align*}
L_{\text{WDAL}} 
&= \max\!\bigl(L_W^2, L_D^2\bigr) \\
&\quad \cdot\; e^{\log L_W - \log \max(L_W, L_D)} \\
&\quad \cdot\; e^{\log L_D - \log \max(L_W, L_D)}.
\end{align*}

\paragraph{Case Analysis:}
We now consider two cases based on which of \(L_W\) or \(L_D\) is larger.

\subparagraph{Case I: \(L_W > L_D\).}
Then
\[
\max(L_W, L_D) = L_W 
\quad\Longrightarrow\quad
M = \log L_W,
\]
and
\[
L_{\text{WDAL}} 
= L_W^2 \, e^{\log L_D - \log L_W}.
\]

\subparagraph{Case II: \(L_D > L_W\).}
Then
\[
\max(L_W, L_D) = L_D 
\quad\Longrightarrow\quad
M = \log L_D,
\]
and
\[
L_{\text{WDAL}} 
= L_D^2 \, e^{\log L_W - \log L_D}.
\]

\paragraph{Generalized Form:}
Both cases can be combined by noting the exponent is 
\(\exp(-|\log L_W - \log L_D|)\). Thus, we obtain:
\begin{align*}
L_{\text{WDAL}} 
&= \max\!\bigl(L_W^2, L_D^2\bigr) \\
&\quad \cdot \exp\!\Bigl(-\bigl|\log L_W - \log L_D\bigr|\Bigr).
\end{align*}

\subsection{Interpretation}
\begin{itemize}
    \item \textbf{Authority Component \(\max(L_W^2, L_D^2)\):}
    Focuses on whichever loss is larger, ensuring the training signals address the most critical error source.
    \item \textbf{Alignment Penalty \(\exp(-|\log L_W - \log L_D|)\):}
    Penalizes mismatches between writer and director, pushing their losses to be consistent in log-space.
\end{itemize}

This completes the derivation and motivational breakdown of the Writer-Director Alignment Loss.

\subsection{Differentiability and Subdifferentiability of \texorpdfstring{\(\max(L_W^2, L_D^2)\)}{max(L\_W\^{}2, L\_D\^{}2)}}

The function \(\max(L_W^2, L_D^2)\) is defined as:
\begin{equation}
\label{eq:max_def}
\max(L_W^2, L_D^2) = 
\begin{cases} 
L_W^2, & \text{if } L_W^2 > L_D^2, \\
L_D^2, & \text{if } L_D^2 > L_W^2.
\end{cases}
\end{equation}
This function is continuous everywhere but introduces a non-differentiable point at \(L_W^2 = L_D^2\). To analyze its differentiability and subdifferentiability, we compute the partial derivatives (with respect to \(L_W\) and \(L_D\)) for \(L_W^2 \neq L_D^2\):

\begin{equation}
\label{eq:max_derivs}
\begin{aligned}
\frac{\partial}{\partial L_W}\max(L_W^2, L_D^2) &=
\begin{cases} 
2L_W, & \text{if } L_W^2 > L_D^2, \\
0,    & \text{if } L_W^2 < L_D^2,
\end{cases}\\[6pt]
\frac{\partial}{\partial L_D}\max(L_W^2, L_D^2) &=
\begin{cases} 
0,    & \text{if } L_W^2 > L_D^2, \\
2L_D, & \text{if } L_W^2 < L_D^2.
\end{cases}
\end{aligned}
\end{equation}

At the point \(L_W^2 = L_D^2\), the derivative is undefined due to the non-smooth transition between the two cases. However, \(\max(L_W^2, L_D^2)\) is \emph{subdifferentiable} at \(L_W^2 = L_D^2\), and its subdifferential can be expressed as:
\begin{equation}
\label{eq:max_subdiff}
\partial\,\max(L_W^2, L_D^2) 
= \Bigl\{\bigl(a \cdot 2L_W,\,(1-a)\cdot 2L_D\bigr).
\end{equation}

This subgradient property ensures that for any \(a \in [0,1]\),

\begin{multline*}
\frac{\partial}{\partial L_W}\max(L_W^2, L_D^2) 
+ \frac{\partial}{\partial L_D}\max(L_W^2, L_D^2) \\
= 2L_W + 2L_D.
\end{multline*}

Notably, \(\max(L_W^2, L_D^2)\) retains the convexity of the \(\max\) operator, ensuring its suitability in optimization methods even with the non-differentiability at \(L_W^2 = L_D^2\). Standard gradient-based approaches can use any valid subgradient from \(\partial\,\max(L_W^2, L_D^2)\) at that point, thereby preserving convergence guarantees in convex optimization frameworks.

\subsection*{Differentiability of \(\exp\!\bigl(-\lvert \log L_W - \log L_D\rvert\bigr)\)}

The term \(\bigl|\log L_W - \log L_D\bigr|\) is piecewise differentiable and can be written as:
\begin{multline*}
\label{eq:abs_term}
|\log L_W - \log L_D| = \\
\begin{cases} 
\log L_W \!-\! \log L_D, & \text{if } \log L_W \!\geq\! \log L_D, \\
\log L_D \!-\! \log L_W, & \text{if } \log L_D \!>\! \log L_W.
\end{cases}
\end{multline*}

Hence, its partial derivatives with respect to \(L_W\) and \(L_D\) are:
\begin{multline*}
\begin{aligned}
\frac{\partial}{\partial L_W}\bigl|\log L_W - \log L_D\bigr|
&= \\
\begin{cases}
\frac{1}{L_W},   & \text{if } \log L_W \geq \log L_D,\\
-\frac{1}{L_W},  & \text{if } \log L_D > \log L_W,
\end{cases} \\[4pt]
\frac{\partial}{\partial L_D}\bigl|\log L_W - \log L_D\bigr|
&= \\
\begin{cases}
-\frac{1}{L_D}, & \text{if } \log L_W \geq \log L_D,\\
\frac{1}{L_D},  & \text{if } \log L_D > \log L_W.
\end{cases}
\end{aligned}
\end{multline*}

By the chain rule, the exponential term
\(\exp\!\bigl(-|\log L_W - \log L_D|\bigr)\) is smooth and differentiable everywhere. Its partial derivatives become:
\[
\begin{aligned}
\frac{\partial}{\partial L_W}\exp\!\bigl(-|\log L_W - \log L_D|\bigr) 
&= \\
-\,\exp\!\bigl(-|\log L_W - \log L_D|\bigr) \\
\times \frac{\partial}{\partial L_W}\bigl|\log L_W - \log L_D\bigr|,\\[6pt]
\frac{\partial}{\partial L_D}\exp\!\bigl(-|\log L_W - \log L_D|\bigr)
&= \\
-\,\exp\!\bigl(-|\log L_W - \log L_D|\bigr) \\
 \times \frac{\partial}{\partial L_D}\bigl|\log L_W - \log L_D\bigr|.
\end{aligned}
\]

At the boundary case \(L_W = L_D\), we have \(\lvert \log L_W - \log L_D\rvert = 0\), which implies:
\begin{multline*}
\exp\!\bigl(-|\log L_W - \log L_D|\bigr) = 1, \\
\frac{\partial}{\partial L_W}\exp\!\bigl(-|\log L_W - \log L_D|\bigr) = 0, \\
\frac{\partial}{\partial L_D}\exp\!\bigl(-|\log L_W - \log L_D|\bigr) = 0.
\end{multline*}
Hence, there is no discontinuity in the exponential term at \(L_W = L_D\). It remains differentiable and can be seamlessly incorporated into gradient-based optimization routines.

\subsection{Theoretical Bounds of the WDAL}

To gain insight into the possible range of the Writer-Director Alignment Loss (WDAL),
recall its final form:
\begin{multline*}
L_{\text{WDAL}} = \\
\max\!\bigl(L_W^2, L_D^2\bigr)\,\exp\!\Bigl(-\bigl|\log L_W - \log L_D\bigr|\Bigr).
\end{multline*}
For \(L_W, L_D > 0\), we can show that this simplifies to \(L_W\,L_D\) except at the boundary \(L_W = L_D\) 
(where it still equals \(L_W \, L_D\)). Consequently:
\[
L_{\text{WDAL}} = L_W \, L_D.
\]
provided \(L_W \neq 0\) and \(L_D \neq 0\). Below, we outline the key boundary behaviors.

\paragraph{Lower Bound.}
If either \(L_W \to 0\) or \(L_D \to 0\), then \(L_W \, L_D \to 0\). Hence,
\[
\lim_{\substack{L_W \to 0 \;\text{or}\; L_D \to 0}} L_{\text{WDAL}} = 0.
\]
However, note that \(\log(0)\) is undefined numerically; in practice, one ensures \(L_W\) and \(L_D\) stay positive or uses a small \(\epsilon\) (e.g., \(10^{-8}\)) to avoid taking the log of zero. Still, from a theoretical standpoint, 
\(\boxed{ \min L_{\text{WDAL}} = 0 }\).

\paragraph{Upper Bound.}
As either \(L_W \to \infty\) or \(L_D \to \infty\), the product \(L_W \, L_D \to \infty\). 
Thus, there is no finite upper bound:
\[
\lim_{\substack{L_W \to \infty \;\text{or}\; L_D \to \infty}} L_{\text{WDAL}} = \infty.
\]
Hence, \(\boxed{ L_{\text{WDAL}} \text{ is unbounded above.}}\)

\paragraph{Overall Range.}
Combining these observations, for \(L_W, L_D \geq 0\), the possible values of 
\(L_{\text{WDAL}}\) lie in the interval \([0,\infty)\). 
In most practical NLP applications, neither loss would \emph{exactly} be zero nor unbounded, 
so \(L_{\text{WDAL}}\) typically occupies a finite, positive range. Nonetheless, the flexibility 
to approach 0 or grow arbitrarily large is crucial for reflecting both 
\emph{complete success} (very small losses) and \emph{catastrophic failure} (very large losses).

\section{Theoretical Justification}

\subsection{Conditioning Reduces Entropy}\label{app_sec:proof_theorem_2}

\begin{theorem}
Let \( X \) and \( Y \) be continuous random variables with joint density \( f_{X,Y}(x,y) \), marginal densities \( f_X(x) \), \( f_Y(y) \), and conditional density \( f_{X|Y}(x|y) \). The differential entropy satisfies:

\[
H(X) \geq H(X|Y),
\]

where \( H(X) \) and \( H(X|Y) \) denote the marginal and conditional differential entropy, respectively. \cite{thomas2006elements}

\end{theorem} 
\begin{proof}

For continuous random variables, differential entropy is defined as:

\begin{multline*}
H(X) = - \int f_X(x) \log f_X(x) dx,\\
H(X|Y) = - \iint f_{X,Y}(x,y) \log f_{X|Y}(x|y) dx dy.
\end{multline*}

Substituting \( f_{X|Y}(x|y) = \frac{f_{X,Y}(x,y)}{f_Y(y)} \) into \( H(X|Y) \), we derive:

\[
H(X|Y) = - \iint f_{X,Y}(x,y) \log \frac{f_{X,Y}(x,y)}{f_Y(y)} dx dy
\]

Expanding the logarithm:
\small
\begin{multline*}
H(X|Y) = - \underbrace{\iint f_{X,Y}(x,y) \log f_{X,Y}(x,y) \,dxdy}_{H(X,Y)} \\
+ \iint f_{X,Y}(x,y) \log f_Y(y) \,dxdy.
\end{multline*}
\normalsize

The second term simplifies using the marginal \( \int f_{X,Y}(x,y) dx = f_Y(y) \):

\begin{multline*}
\iint f_{X,Y}(x,y) \log f_Y(y) dx dy = \\ \int f_Y(y) \log f_Y(y) dy = - H(Y).
\end{multline*}

Thus,

\[
H(X|Y) = H(X,Y) - H(Y).
\]

To show \( H(X) \geq H(X|Y) \), we invoke the non-negativity of the Kullback-Leibler (KL) divergence:

\begin{multline*}
D_{\text{KL}}(f_{X,Y} \| f_X f_Y) = \\
\iint f_{X,Y}(x,y) \log \frac{f_{X,Y}(x,y)}{f_X(x) f_Y(y)} dx dy \geq 0.
\end{multline*}

Expanding the integrand:

\begin{multline*}
D_{\text{KL}} = \iint f_{X,Y}(x,y) \log f_{X,Y}(x,y) dx dy \\
- \iint f_{X,Y}(x,y) \log f_X(x) dx dy - \\
\iint f_{X,Y}(x,y) \log f_Y(y) dx dy.
\end{multline*}

Recognizing the entropy terms:

\begin{multline*}
D_{\text{KL}} = - H(X,Y) + H(X) + H(Y) \geq 0 \\ 
\implies H(X) + H(Y) \geq H(X,Y).
\end{multline*}

Substituting \( H(X,Y) = H(X|Y) + H(Y) \) into the inequality:

\[
H(X) \geq H(X|Y).
\]

\end{proof}

\subsection{Monotonicity of Conditional Entropy} \label{app_sec:proof_theorem_3}.

\begin{theorem}
Let $X, Y, Z$ be continuous random variables. Differential entropy satisfies:
\[
H(X|Y) \geq H(X|Y,Z),
\]
with equality if $X \perp Z | Y$. This generalizes to 
\begin{multline*}
H(X|Y_1) \geq H(X|Y_1, Y_2) \geq \cdots \\ 
\geq H(X|Y_1, \dots, Y_n).
\end{multline*}

\end{theorem} 

\begin{proof}
    
The conditional differential entropies are defined as:
\begin{multline*}
H(X|Y) = - \iint f_Y(y) f_{X|Y}(x|y) \\
\log f_{X|Y}(x|y) \,dxdy,
\end{multline*}
\begin{multline*}
H(X|Y,Z) = - \iiint f_{Y,Z}(y,z) f_{X|Y,Z}(x|y,z) \\ 
\log f_{X|Y,Z}(x|y,z) \,dxdydz.
\end{multline*}

The marginal conditional density relates to the joint via:
\begin{multline*}
f_{X|Y}(x|y) = \int f_{X|Y,Z}(x|y,z) f_{Z|Y}(z|y) \,dz.
\end{multline*}

Substituting this into $H(X|Y)$:
\begin{multline*}
H(X|Y) = - \iint f_Y(y) . \\
\underbrace{\left[ \int f_{X|Y,Z}(x|y,z) f_{Z|Y}(z|y) \,dz \right] }_{\textstyle f_{X|Y}(x|y)} \\
\log \left( \int f_{X|Y,Z}(x|y,z) f_{Z|Y}(z|y) \,dz \right) \,dxdy.
\end{multline*}

\textit{Apply Jensen's inequality}

Using the convexity of $-\log(\cdot)$ and Jensen's inequality:
\begin{multline*}
-\log \left( \int f_{X|Y,Z}(x|y,z) f_{Z|Y}(z|y) \,dz \right) \leq \\
-\int f_{Z|Y}(z|y) \log f_{X|Y,Z}(x|y,z) \,dz.
\end{multline*}

Substituting this bound:
\begin{multline*}
H(X|Y) \geq - \iiint f_Y(y) f_{Z|Y}(z|y) \\
f_{X|Y,Z}(x|y,z) \log f_{X|Y,Z}(x|y,z) \,dxdydz.
\end{multline*}

Simplifying via $f_Y(y) f_{Z|Y}(z|y) = f_{Y,Z}(y,z)$:
\begin{multline*}
H(X|Y) \geq - \iiint f_{Y,Z}(y,z) \\ 
f_{X|Y,Z}(x|y,z) \log f_{X|Y,Z}(x|y,z) \,dxdydz = \\
H(X|Y,Z).
\end{multline*}

\textbf{Generalization.} Iteratively applying this result gives:
\begin{multline*}
H(X|Y_1) \geq H(X|Y_1, Y_2) \geq \cdots \\
\geq H(X|Y_1, \dots, Y_n).
\end{multline*}
\end{proof}

\subsection{Mutual Information Decreases Under Deterministic Compression}\label{app_sec:proof_theorem_1}

Let \( \mathbf{X} \in \mathbb{R}^{n \times d} \) denote the input with \( n \) tokens and \( d \)-dimensional embeddings, and let \( \mathbf{Z} \in \mathbb{R}^{n \times d} \) represent the final hidden representation. Here \( \mathbf{Z_p} \in \mathbb{R}^{d}\) as a pooled representation. $\mathbf{Y}$ is output.

We aim to show that the generator provides better predictions than the predictor, i.e., 
\begin{equation} I(\mathbf{Y}; \mathbf{Z}) \geq I(\mathbf{Y}; \mathbf{Z_p}).
\end{equation}

According to  the Data Processing Inequality (DPI)~\cite{beaudry2011intuitive}, this process of predictor can be expressed as
\begin{equation}
\mathbf{X}  \to \mathbf{Z} \to \mathbf{Z_p} \to \mathbf{Y}.
\end{equation}

Similarly for the generator
\begin{equation}
\mathbf{X}  \to \mathbf{Z} \to \mathbf{Y}.
\end{equation}

Now, from the definition of mutual information in the Information Bottleneck (IB) framework, we know
\begin{equation}\label{eq::mutual_Z_p}
I(\mathbf{Y}; \mathbf{Z_p}) = H(\mathbf{Y}) - H(\mathbf{Y}|\mathbf{Z_p}),
\end{equation}
and similarly
\begin{equation} \label{eq::mutual_Z}
I(\mathbf{Y}; \mathbf{Z}) = H(\mathbf{Y}) - H(\mathbf{Y}|\mathbf{Z}).
\end{equation}

To prove that the the generator is  better than the predictor in predicting $\mathbf{Y}$, we need to show 
\begin{equation}
I(\mathbf{Y}; \mathbf{Z}) \geq I(\mathbf{Y}; \mathbf{Z_p}).
\end{equation}

From \eqref{eq::mutual_Z_p} and \eqref{eq::mutual_Z}, this reduces to proving
\[
 H(\mathbf{Y}|\mathbf{Z}) \leq H(\mathbf{Y}|\mathbf{Z_p}).
\]

Now in order to proof this, we can consider two real scenario how we are getting $\mathbf{Z_p}$. One  common practice is  pooling where output $\mathbf{Z_p}$ is considered the first special token representation of model and another case is $\mathbf{Z_p}$ which will be considered as the means of all tokens representation. 

\textbf{When \(\mathbf{Z}_p\) is the representation of the first special token:}\\
In this case, we can define
\[
  \mathbf{Z} \;=\; \bigl(\mathbf{Z}_p,\, \mathbf{Z}_1,\, \mathbf{Z}_2,\, \dots,\, \mathbf{Z}_{n-1}\bigr) \;\in\; \mathbb{R}^{n \times d}.
\]
Then we have
\[
  H(\mathbf{Y}\,\big\vert\, \mathbf{Z})
  \;=\;
  H\bigl(\mathbf{Y}\,\big\vert\, \mathbf{Z}_p, \mathbf{Z}_1, \mathbf{Z}_2, \dots, \mathbf{Z}_{n-1}\bigr).
\]

 Now following  the principle of conditional entropy and the previous theorem (\ref{app_sec:proof_theorem_2} and \ref{app_sec:proof_theorem_3}):
\begin{multline*}
H(X|Y_1, Y_2, \dots, Y_n) \leq H(X|Y_1, Y_2, \dots, Y_{n-1}) \\
\leq \dots \leq H(X|Y_1),
\end{multline*}
it follows that:
\[  H(\mathbf{Y}| \mathbf{Z_p}, \mathbf{Z_1}, \mathbf{Z_2}, \cdots \mathbf{Z_{n-1}}) \leq H(\mathbf{Y}|\mathbf{Z_p}). \]
Therefore:
\[  H(\mathbf{Y}|\mathbf{Z}) \leq H(\mathbf{Y}|\mathbf{Z_p}). \]

Combining this result with the mutual information definitions, we conclude:
\[ I(\mathbf{Y}; \mathbf{Z}) \geq I(\mathbf{Y}; \mathbf{Z_p}). \]

\textbf{When $\mathbf{Z_p}$ is the mean of $\mathbf{Z}$:} From the definition of conditional entropy
\begin{multline*}
  H(\mathbf{Y} \mid \mathbf{Z}) 
  \;=\; \\ \mathbb{E}_{p(\mathbf{Z})} \Bigl[ -\!\!\int p(\mathbf{Y} \mid \mathbf{Z})\,\log p(\mathbf{Y} \mid \mathbf{Z})\,d\mathbf{Y} \Bigr].
\end{multline*}
\normalsize
Similarly,
\begin{multline*}
H(\mathbf{Y} \mid \mathbf{Z}_p) 
= \\ \mathbb{E}_{p(\mathbf{Z}_p)} \Bigl[ 
    -\!\!\int p(\mathbf{Y} \mid \mathbf{Z}_p) 
    \log p(\mathbf{Y} \mid \mathbf{Z}_p) 
    d\mathbf{Y} 
\Bigr].
\end{multline*}

Hence, we need to show
\begin{multline*}
  \mathbb{E}_{p(\mathbf{Z}_p)} \Bigl[ -\!\!\int p(\mathbf{Y} \mid \mathbf{Z}_p)\,\log p(\mathbf{Y} \mid \mathbf{Z}_p)\,d\mathbf{Y} \Bigr]
  \;\;\ge\;\; \\
  \mathbb{E}_{p(\mathbf{Z})} \Bigl[ -\!\!\int p(\mathbf{Y} \mid \mathbf{Z})\,\log p(\mathbf{Y} \mid \mathbf{Z})\,d\mathbf{Y} \Bigr].
\end{multline*}

By the law of total probability (marginalization), for each fixed value \(\mathbf{z}_p\) of \(\mathbf{Z}_p\):
\begin{multline*}
  p(\mathbf{Y} \mid \mathbf{Z}_p = \mathbf{z}_p)
  \;=\; \\
  \int p(\mathbf{Y} \mid \mathbf{Z})\,p(\mathbf{Z} \mid \mathbf{Z}_p = \mathbf{z}_p)\,d\mathbf{Z}.
\end{multline*}
Hence 
\[
  p(\mathbf{Y} \mid \mathbf{Z}_p) 
  \;=\; 
  \int p(\mathbf{Y} \mid \mathbf{Z})\,p(\mathbf{Z} \mid \mathbf{Z}_p)\,d\mathbf{Z}.
\]
Substitute the above mixture form into the conditional entropy expression:
\begin{multline*}
H(\mathbf{Y} \mid \mathbf{Z}_p) 
= \\
\mathbb{E}_{p(\mathbf{Z}_p)} \Bigl[
    -\!\!\int 
      p(\mathbf{Y} \mid \mathbf{Z}_p)\, 
      \log p(\mathbf{Y} \mid \mathbf{Z}_p)\,
    d\mathbf{Y}
\Bigr] 
\\[6pt]
= \mathbb{E}_{p(\mathbf{Z}_p)} \Bigl[
    -\!\!\int 
      \Bigl(\int p(\mathbf{Y} \mid \mathbf{Z})\, 
      p(\mathbf{Z} \mid \mathbf{Z}_p)\,
    d\mathbf{Z}\Bigr) \\
    \log \Bigl(\int p(\mathbf{Y} \mid \mathbf{Z})\, 
      p(\mathbf{Z} \mid \mathbf{Z}_p)\,
    d\mathbf{Z}\Bigr)
    d\mathbf{Y}
\Bigr].
\end{multline*}

Note that the function \(\,-x \log x\) is \emph{concave} for \(x > 0\). 
Equivalently, the Shannon entropy
\[
  H(p) \;=\; -\!\int p(\mathbf{y})\,\log p(\mathbf{y})\,d\mathbf{y}
\]
is a \emph{concave functional} in \(p\).  
Therefore, for the mixture 
\(\,\int p(\mathbf{Y} \mid \mathbf{Z})\,p(\mathbf{Z} \mid \mathbf{Z}_p)\,d\mathbf{Z},\)
we have:
\small
\begin{multline*}
-\!\int
    \Bigl(\!\int p(\mathbf{Y} \mid \mathbf{Z})\,p(\mathbf{Z} \mid \mathbf{Z}_p)\,d\mathbf{Z} \Bigr) \\
    \log 
    \Bigl(\!\int p(\mathbf{Y} \mid \mathbf{Z})\,p(\mathbf{Z} \mid \mathbf{Z}_p)\,d\mathbf{Z} \Bigr)
  \, d\mathbf{Y} 
\\[6pt]
\ge
  \iint p(\mathbf{Z} \mid \mathbf{Z}_p)\,p(\mathbf{Y} \mid \mathbf{Z}) \,
  \Bigl[-\,\log p(\mathbf{Y} \mid \mathbf{Z})\Bigr]
  \,d\mathbf{Y}\,d\mathbf{Z}
\\[6pt]
=
  \int p(\mathbf{Z} \mid \mathbf{Z}_p)\Bigl[
    -\!\!\int p(\mathbf{Y} \mid \mathbf{Z})\,\log p(\mathbf{Y} \mid \mathbf{Z})\,d\mathbf{Y}
  \Bigr] d\mathbf{Z}.
\end{multline*}
\normalsize

Thus, inside the expectation over \(\mathbf{Z}_p\), we get using Jensen's inequality:
\small
\begin{multline*}
-\!\!\int 
    p(\mathbf{Y} \mid \mathbf{Z}_p)\,\log p(\mathbf{Y} \mid \mathbf{Z}_p)
  \, d\mathbf{Y}
\\[6pt]
\ge
  \int p(\mathbf{Z} \mid \mathbf{Z}_p) 
    \Bigl[
      -\!\!\int p(\mathbf{Y} \mid \mathbf{Z})\,\log p(\mathbf{Y} \mid \mathbf{Z})\,d\mathbf{Y}
    \Bigr] 
  d\mathbf{Z}.
\end{multline*}
\normalsize

Taking the expectation w.r.t.\ \(p(\mathbf{Z}_p)\) then yields
\begin{multline*}
H(\mathbf{Y} \mid \mathbf{Z}_p)
= \mathbb{E}_{p(\mathbf{Z}_p)} \Bigl[
    -\!\!\int 
      p(\mathbf{Y} \mid \mathbf{Z}_p)\,\\
      \log p(\mathbf{Y} \mid \mathbf{Z}_p)\,d\mathbf{Y}
  \Bigr]
\\[5pt]
\ge 
\mathbb{E}_{p(\mathbf{Z}_p)} \Bigl[
    \int p(\mathbf{Z} \mid \mathbf{Z}_p)\Bigl(
      -\!\!\int p(\mathbf{Y} \mid \mathbf{Z})\, \\
      \log p(\mathbf{Y} \mid \mathbf{Z})\,d\mathbf{Y}
    \Bigr) d\mathbf{Z}
  \Bigr].
\end{multline*}

Using the law of total expectation, we get
\small
\begin{multline*}
\int p(\mathbf{Z} \mid \mathbf{Z}_p)\Bigl[
    -\!\!\int p(\mathbf{Y} \mid \mathbf{Z})\,\log p(\mathbf{Y} \mid \mathbf{Z})\,d\mathbf{Y}
  \Bigr] d\mathbf{Z}
\\
= 
\mathbb{E}_{p(\mathbf{Z} \mid \mathbf{Z}_p)}\Bigl[ H(\mathbf{Y} \mid \mathbf{Z}) \Bigr].
\end{multline*}
\normalsize
Hence
\begin{multline*}
H(\mathbf{Y} \mid \mathbf{Z}_p)
\;\;\ge\;\;
\mathbb{E}_{p(\mathbf{Z}_p)} \Bigl[
    \mathbb{E}_{p(\mathbf{Z} \mid \mathbf{Z}_p)}\bigl[H(\mathbf{Y} \mid \mathbf{Z})\bigr]
  \Bigr]
\\
\;=\;
\mathbb{E}_{p(\mathbf{Z})} \bigl[H(\mathbf{Y} \mid \mathbf{Z})\bigr].
\end{multline*}

where the last equality uses the law of total expectation 
(\(\mathbb{E}_{\mathbf{Z}_p}[\mathbb{E}_{\mathbf{Z} \mid \mathbf{Z}_p}(\cdot)] = \mathbb{E}_{\mathbf{Z}}[\cdot]\)) 
and the fact that \(\mathbf{Z}_p\) is a deterministic function of \(\mathbf{Z}\).  
Thus,
\[
  H(\mathbf{Y} \mid \mathbf{Z}_p)
  \;\;\ge\;\;
  H(\mathbf{Y} \mid \mathbf{Z}).
\]

Combining this result with the mutual information definitions, we conclude:
\[ I(\mathbf{Y}; \mathbf{Z}) \geq I(\mathbf{Y}; \mathbf{Z_p}). \]

\begin{figure}[!t]
    \centering
    \includegraphics[width=0.9980\linewidth]{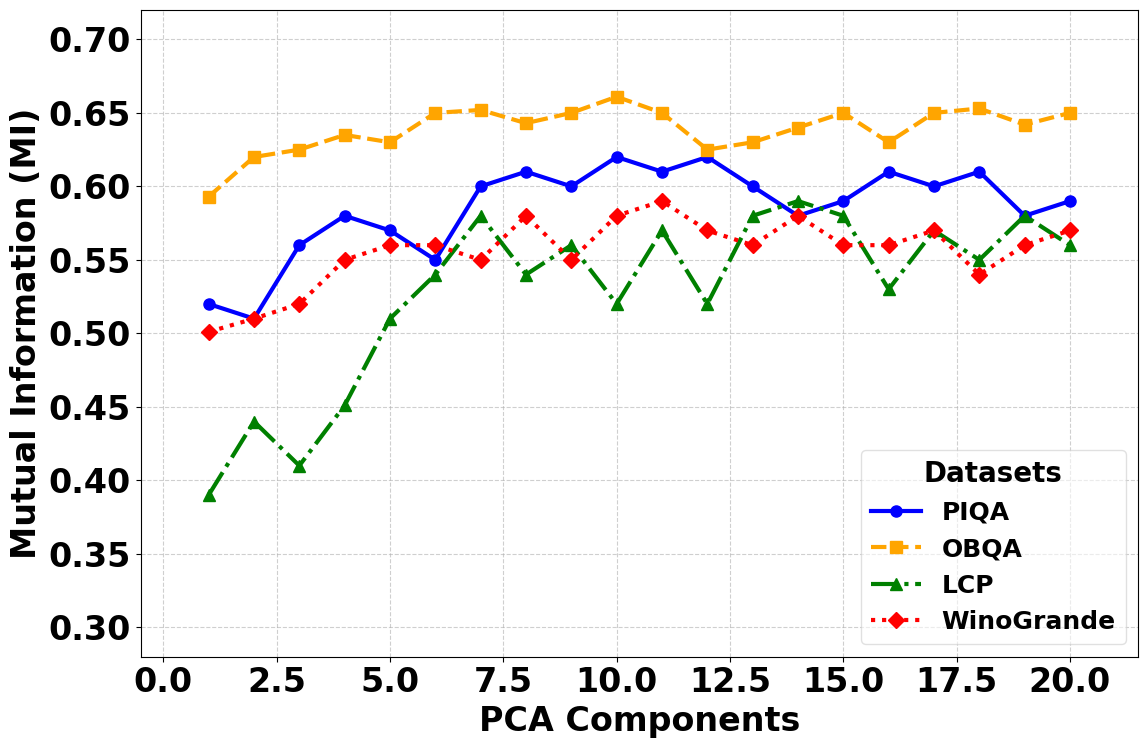}
\caption{Effect of PCA-based dimensionality reduction on mutual information.}
\label{fig:PCA_MI}
\end{figure}

\section{More Ablation Studies} \label{appen:abalation}
\subsection{PCA Components vs. Mutual Information:}

To estimate the mutual information for generation (as discussed in Section~\ref{sec:generation-superior}), we apply Principal Component Analysis (PCA) to reduce the dimensionality of the token representations. Instead of using the full representation of \(\mathbf{Z}\) (which has dimensions \(n \times d\)), we reduce it to \(k\) PCA components. This helps lower the computational cost while still retaining the most important information.

Figure~\ref{fig:PCA_MI} shows how the number of PCA components affects mutual information estimation. Initially, as we increase the number of components, the mutual information improves because more task-relevant details are preserved. However, after about 10 to 12 components, the mutual information levels off and adding more components doesn't significantly improve the results. For example, in the OBQA dataset, mutual information increases from 0.593 with just 1 component to 0.661 with 10 components, but stabilizes beyond that point. We observe similar patterns across other datasets.

\section{Implementation and Hyperparameter Details}
\label{sec:impl}

We design the \emph{task adapter} $\mathcal{T}$ in \textsc{PredGen} to map the model-generated token sequence $\tilde{\mathbf{Y}} = [\tilde{\mathbf{Y}}_1, \tilde{\mathbf{Y}}_2, \dots, \tilde{\mathbf{Y}}_m]$ into the final prediction $\tilde{\mathbf{P}}$. Below, we distinguish classification and regression setups.

\subsection{Classification Task Adapter.}
For classification, let $\tilde{\mathbf{Y}}$ be the autoregressively generated tokens. We then define
\[
   \tilde{\mathbf{P}} \;=\; \mathcal{T}(\tilde{\mathbf{Y}}) \;=\; \mathrm{softmax}\Bigl(W\,\bigl[\mathrm{CLS}(\tilde{\mathbf{Y}})\bigr]\Bigr),
\]
where $W$ is a trainable parameter matrix and $\mathrm{CLS}(\cdot)$ denotes extracting a  “classification” embedding. We use the standard cross-entropy loss $\mathcal{L}_{\text{CE}}$ for both the generator (token-level) and the classifier (label-level). Since both losses live in the same space (cross-entropy), our Writer-Director Alignment Loss (\textsc{WDAL}) remains stable.

\subsection{Regression Task Adapter with Ordered Penalty.}
When the final output must be a continuous value (e.g., $0.75$ in a Semantic Textual Similarity task), representing it as tokens \texttt{[0, ., 7, 5]} can introduce ambiguity: cross-entropy penalizes each token without directly encoding how numerically close $\tilde{\mathbf{Y}}$ is to $\mathbf{Y}$. To address this, we introduce an \emph{ordered penalty} that magnifies errors in more significant token positions.

Formally, let $\mathbf{Y} = [\mathbf{Y}_1, \mathbf{Y}_2, \dots, \mathbf{Y}_m]$ be the ground-truth token sequence (e.g., digits of a decimal number) and $\tilde{\mathbf{Y}} = [\tilde{\mathbf{Y}}_1, \tilde{\mathbf{Y}}_2, \dots, \tilde{\mathbf{Y}}_m]$ be the predicted tokens. Define a monotonically decreasing penalty vector 
\[
   \boldsymbol{\alpha} \;=\; [\alpha_1, \alpha_2, \dots, \alpha_m]
\]
\[
   \quad \text{where } \alpha_1 \ge \alpha_2 \ge \dots \ge \alpha_m > 0.
\]
Each position $i$ is then weighted by $\alpha_i$ to reflect its significance. We modify the token-level cross-entropy $\mathcal{L}_{\text{CE}}( \tilde{y}_i, y_i )$ to:
\[
  L_D =  \mathcal{L}_{\text{ord}}\bigl(\tilde{\mathbf{Y}}, \mathbf{Y}\bigr)
   \;=\;
   \sum_{i=1}^{m} \alpha_i \,\mathcal{L}_{\text{CE}}\bigl(\tilde{\mathbf{Y}}_i, \mathbf{Y}_i \bigr).
\]
In practice, we choose $\boldsymbol{\alpha}$ so that early-token mismatches (e.g., the integer or first decimal digit) incur substantially larger penalties compared to later tokens. For example, if $\mathbf{Y}$ has length $4$, we might set $\alpha_1=1.67, \alpha_2=1.33, \alpha_3=1.01, \alpha_4=1.00$, ensuring that predicting \texttt{0.76} (only off by one digit in the third position) has lower cost than predicting \texttt{1.75} (off in the first digit).

\subsection{WDAL Consistency.}
Both the “writer” (generative cross-entropy) and the “director” (classification or regression output) losses must be consistently derived from the same token sequence to maintain stability in WDAL. For classification, this alignment is direct since both losses are cross-entropy. For regression, the ordered penalty $\mathcal{L}_{\text{ord}}$—which is still a sum of token-level cross-entropies with position-dependent scaling—remains compatible with the generative objective, preventing conflicts that arise from mixing cross-entropy (token-level) and MSE/MAE (numerical-value-level).

In all cases, hyperparameters (e.g., learning rate, penalty coefficients $\alpha_i$, and scheduling parameters for training) are tuned to optimize final validation performance while preserving the coherence of generation and prediction. The hyperparameters details of datasets are given in Table~\ref{tab:hyperparameters}.

\begin{table*}[ht]
\centering
\resizebox{\textwidth}{!}{%
\begin{tabular}{l|c|c|c|c|c|c|c}
\hline
\rowcolor{gray!20}
\textbf{Dataset} & \textbf{Learning Rate} & \textbf{Batch Size} & \textbf{Grad. Accum. Steps} & \textbf{Epochs} & \textbf{Warmup Steps} & \textbf{LR Scheduler Type} & \textbf{Max Steps for Sampling} \\
\hline
\rowcolor[HTML]{EAF3FA} \textbf{CLEAR} & 2e-3 & 3 & 4 & 25 & 200 & cosine & 500 \\
\rowcolor[HTML]{EAF3FA} \textbf{Humicroedit} & 2e-3 & 6 & 4 & 10 & 200 & cosine & 500 \\
\rowcolor[HTML]{EAF3FA} \textbf{LCP} & 2e-3 & 1 & 5 & 10 & 200 & cosine & 500 \\
\rowcolor[HTML]{EAF3FA} \textbf{SICK} & 2e-3 & 6 & 4 & 10 & 200 & cosine & 500 \\
\rowcolor[HTML]{EAF3FA} \textbf{stsb} & 2e-3 & 5 & 4 & 20 & 200 & cosine & 500 \\
\rowcolor[HTML]{EAF3FA} \textbf{WASSA} & 2e-3 & 5 & 4 & 10 & 200 & cosine & 500 \\
\rowcolor[HTML]{F8EDEB} \textbf{BoolQ} & 2e-3 & 10 & 3 & 10 & 500 & cosine & 500 \\
\rowcolor[HTML]{F8EDEB} \textbf{PIQA} & 2e-3 & 10 & 3 & 10 & 500 & cosine & 1000 \\
\rowcolor[HTML]{F8EDEB} \textbf{SIQA} & 2e-3 & 10 & 3 & 10 & 500 & cosine & 500 \\
\rowcolor[HTML]{F8EDEB} \textbf{HellaSwag} & 2e-3 & 10 & 3 & 10 & 500 & cosine & 1000 \\
\rowcolor[HTML]{F8EDEB} \textbf{WinoGrande} & 2e-3 & 10 & 3 & 10 & 500 & cosine & 1000 \\
\rowcolor[HTML]{F8EDEB} \textbf{ARC-e} & 2e-3 & 10 & 3 & 10 & 500 & cosine & 500 \\
\rowcolor[HTML]{F8EDEB} \textbf{ARC-c} & 2e-3 & 10 & 3 & 10 & 500 & cosine & 500 \\
\rowcolor[HTML]{F8EDEB} \textbf{OBQA} & 2e-3 & 10 & 3 & 10 & 500 & cosine & 1000 \\
\hline
\end{tabular}%
}
\caption{Hyperparameters for Different Datasets}
\label{tab:hyperparameters}
\end{table*}

\section{Dataset Description} \label{sec:dataset}

The details of the used datasets in this work have been described in Table~\ref{tab:datasets}.

\begin{table*}[ht]
\centering
\resizebox{\textwidth}{!}{%
\begin{tabular}{l|l|l|p{9cm}}
\hline
\rowcolor{gray!20} 
\textbf{Dataset} & \textbf{Task Type} & \textbf{Domain} & \textbf{Description} \\
\hline
\rowcolor[HTML]{EAF3FA} \textbf{BoolQ} \cite{clark2019boolq} & Classification & Reading Comprehension & A binary question-answering dataset where each naturally occurring question requires a yes/no response based on a Wikipedia passage. \\
\rowcolor[HTML]{EAF3FA} \textbf{PIQA} \cite{bisk2020piqa} & Classification & Physical Commonsense & Evaluates models' ability to perform physical reasoning by selecting the most plausible solution to a given problem scenario. \\
\rowcolor[HTML]{EAF3FA} \textbf{SIQA} \cite{sap2019socialiqa} & Classification & Social Intelligence & Tests reasoning about human intentions, emotions, and social interactions by predicting the most likely response to a given situation. \\
\rowcolor[HTML]{EAF3FA} \textbf{HellaSwag} \cite{zellers2019hellaswag} & Classification & Commonsense Reasoning & Challenges models to predict the most logically coherent continuation of a given textual scenario. \\
\rowcolor[HTML]{EAF3FA} \textbf{WinoGrande} \cite{sakaguchi2021winogrande} & Classification & Coreference Resolution & A large-scale dataset designed to assess pronoun resolution by leveraging commonsense knowledge. \\
\rowcolor[HTML]{EAF3FA} \textbf{ARC-e} \cite{clark2018think} & Classification & Science QA & The \textit{easy} subset of the AI2 Reasoning Challenge (ARC), featuring relatively straightforward multiple-choice science questions. \\
\rowcolor[HTML]{EAF3FA} \textbf{ARC-c} \cite{clark2018think} & Classification & Science QA & The \textit{challenge} subset of ARC, containing more difficult questions requiring complex reasoning and external knowledge. \\
\rowcolor[HTML]{EAF3FA} \textbf{OBQA} \cite{mihaylov2018can} & Classification & Open-Book QA & Evaluates the ability to answer multiple-choice science questions using a predefined set of knowledge facts and external world knowledge. \\
\rowcolor[HTML]{F8EDEB} \textbf{WASSA} \cite{vinayakumar2017deepcybernet} & Regression & Sentiment Analysis & A dataset for emotion intensity and sentiment prediction, focusing on fine-grained sentiment classification. \\
\rowcolor[HTML]{F8EDEB} \textbf{SICK} \cite{marelli-etal-2014-sick} & Regression & Semantic Similarity & The Sentences Involving Compositional Knowledge (SICK) dataset, used for sentence similarity assessment and textual entailment tasks. \\
\rowcolor[HTML]{F8EDEB} \textbf{STSB} \cite{cer2017semeval} & Regression & Sentence Similarity & A benchmark for measuring semantic textual similarity, where sentences are scored based on their degree of similarity. \\
\rowcolor[HTML]{F8EDEB} \textbf{LCP} \cite{shardlow2020complex} & Regression & Lexical Complexity & A dataset designed to predict the perceived complexity of words in context. \\
\rowcolor[HTML]{F8EDEB} \textbf{CLEAR} \cite{crossley2023large} & Regression & Reasoning Complexity & Evaluates the difficulty level of reasoning tasks, quantifying cognitive complexity in natural language understanding. \\
\rowcolor[HTML]{F8EDEB} \textbf{Humicroedit} \cite{hossain2019president} & Regression & Humor Perception & A dataset from SemEval that assesses humor perception by analyzing minor text modifications (micro-edits). \\
\rowcolor[HTML]{DFF0D8} \textbf{MultiArith} \cite{roy2015reasoning} & Arithmetic & Mathematics & A dataset focusing on multi-step arithmetic word problems requiring reasoning across multiple operations. \\
\rowcolor[HTML]{DFF0D8} \textbf{GSM8K} \cite{cobbe2021training} & Arithmetic & Mathematics & A high-quality dataset of challenging grade school math word problems for benchmarking reasoning capabilities. \\
\rowcolor[HTML]{DFF0D8} \textbf{AddSub} \cite{hosseini2014learning} & Arithmetic & Mathematics & A dataset designed for evaluating models' ability to solve arithmetic word problems involving addition and subtraction. \\
\rowcolor[HTML]{DFF0D8} \textbf{SingleEq} \cite{koncel2015parsing} & Arithmetic & Mathematics & Contains single-equation arithmetic problems where the goal is to predict the correct numerical result. \\
\rowcolor[HTML]{DFF0D8} \textbf{SVAMP} \cite{patel2021nlp} & Arithmetic & Mathematics & A dataset that introduces variations in arithmetic problems to test the robustness of mathematical reasoning models. \\
\hline
\end{tabular}%
}
\caption{Overview of benchmark datasets used for classification, regression, and arithmetic reasoning tasks.}
\label{tab:datasets}
\end{table*}

\section{Additional Experiments}

Table~\ref{tab:math_performance} presents the results of our math reasoning evaluation. Following \citet{hu2023llm}, we first train on the \textit{math-10k} dataset and then evaluate performance on various math reasoning benchmarks. The reported results use an exact match criterion, where a prediction is considered correct if the difference from the ground truth is less than 0.0001.

We observe that traditional pooling-based classification performs poorly, primarily because it relies on a classifier layer to predict numerical values. Since classification models are not optimized for generating precise numerical outputs, achieving an exact match is extremely rare. In contrast, PredGen consistently achieves higher accuracy across all benchmarks, demonstrating that token-level generation better preserves numerical precision and reasoning ability.

\begin{table*}[htp!]
\centering
\resizebox{\textwidth}{!}{%
\begin{tabular}{|l|l|l|c|c|c|c|c|c|c|c|}
\rowcolor{gray!20}
\textbf{Model}     & \textbf{PEFT}    & \textbf{Method}    & \textbf{MultiArith} & \textbf{GSM8K} & \textbf{AddSub} & \textbf{SingleEq} & \textbf{SVAMP} & \textbf{Avg.}\\ \hline
\multirow{12}{*}{\textbf{Mistral-7B}} 
                    & \cellcolor[HTML]{EAF3FA}      MELoRA              & \cellcolor[HTML]{EAF3FA}    Predictor           & \cellcolor[HTML]{EAF3FA}    5.430          & \cellcolor[HTML]{EAF3FA}    0.032          & \cellcolor[HTML]{EAF3FA}    2.890           & \cellcolor[HTML]{EAF3FA}    3.710               & \cellcolor[HTML]{EAF3FA}    0.986    & \cellcolor[HTML]{EAF3FA}    2.610  \\ 
                    & \cellcolor[HTML]{EAF3FA}                 & \cellcolor[HTML]{EAF3FA}    Generator           & \cellcolor[HTML]{EAF3FA}    75.22           & \cellcolor[HTML]{EAF3FA}    40.22          & \cellcolor[HTML]{EAF3FA}    71.64          & \cellcolor[HTML]{EAF3FA}    71.73               & \cellcolor[HTML]{EAF3FA}    57.12  & \cellcolor[HTML]{EAF3FA}    63.19  \\ 
                    & \cellcolor[HTML]{D0E7F7}                    & \cellcolor[HTML]{D0E7F7}    PredGen           & \cellcolor[HTML]{D0E7F7}    76.47           & \cellcolor[HTML]{D0E7F7}    42.59          & \cellcolor[HTML]{D0E7F7}    73.43          & \cellcolor[HTML]{D0E7F7}    72.48               & \cellcolor[HTML]{D0E7F7}    58.92  & \cellcolor[HTML]{D0E7F7}    64.78\\ 
                    & \cellcolor[HTML]{EAF3FA}    LoRA-FA           & \cellcolor[HTML]{EAF3FA}    Predictor           & \cellcolor[HTML]{EAF3FA}    4.530           & \cellcolor[HTML]{EAF3FA}    0.030          & \cellcolor[HTML]{EAF3FA}    2.170         & \cellcolor[HTML]{EAF3FA}    3.890              & \cellcolor[HTML]{EAF3FA}    1.030  & \cellcolor[HTML]{EAF3FA}    2.330  \\ 
                    & \cellcolor[HTML]{EAF3FA}                      & \cellcolor[HTML]{EAF3FA}    Generator           & \cellcolor[HTML]{EAF3FA}    75.83          & \cellcolor[HTML]{EAF3FA}    40.11         & \cellcolor[HTML]{EAF3FA}    73.43         & \cellcolor[HTML]{EAF3FA}    70.37              & \cellcolor[HTML]{EAF3FA}    56.84     & \cellcolor[HTML]{EAF3FA}     63.32   \\ 
                    & \cellcolor[HTML]{D0E7F7}                     & \cellcolor[HTML]{D0E7F7}    PredGen            & \cellcolor[HTML]{D0E7F7}    77.53           & \cellcolor[HTML]{D0E7F7}    42.48          & \cellcolor[HTML]{D0E7F7}    75.14          & \cellcolor[HTML]{D0E7F7}    72.48               & \cellcolor[HTML]{D0E7F7}    57.38        & \cellcolor[HTML]{D0E7F7}     65.00    \\ 
                    & \cellcolor[HTML]{EAF3FA}      MoSLoRA          & \cellcolor[HTML]{EAF3FA}    Predictor           & \cellcolor[HTML]{EAF3FA}    4.470           & \cellcolor[HTML]{EAF3FA}    0.082         & \cellcolor[HTML]{EAF3FA}    3.280          & \cellcolor[HTML]{EAF3FA}    4.890               & \cellcolor[HTML]{EAF3FA}    0.894   & \cellcolor[HTML]{EAF3FA}    2.720  \\ 
                    & \cellcolor[HTML]{EAF3FA}                    & \cellcolor[HTML]{EAF3FA}    Generator           & \cellcolor[HTML]{EAF3FA}    74.28           & \cellcolor[HTML]{EAF3FA}    41.48          & \cellcolor[HTML]{EAF3FA}    70.34          & \cellcolor[HTML]{EAF3FA}    71.89               & \cellcolor[HTML]{EAF3FA}    58.06    & \cellcolor[HTML]{EAF3FA}     63.21        \\ 
                    & \cellcolor[HTML]{D0E7F7}                    & \cellcolor[HTML]{D0E7F7}    PredGen            & \cellcolor[HTML]{D0E7F7}    75.45           & \cellcolor[HTML]{D0E7F7}    43.68         & \cellcolor[HTML]{D0E7F7}    72.68          & \cellcolor[HTML]{D0E7F7}    71.11               & \cellcolor[HTML]{D0E7F7}    57.92   & \cellcolor[HTML]{D0E7F7}    64.17  \\ 
                    & \cellcolor[HTML]{EAF3FA}     Propulsion              & \cellcolor[HTML]{EAF3FA}    Predictor           & \cellcolor[HTML]{EAF3FA}    5.840           & \cellcolor[HTML]{EAF3FA}    0.103  & \cellcolor[HTML]{EAF3FA}    2.790          & \cellcolor[HTML]{EAF3FA}     88.21               & \cellcolor[HTML]{EAF3FA}    0.837       & \cellcolor[HTML]{EAF3FA}     19.56 \\ 
                    & \cellcolor[HTML]{EAF3FA}                     & \cellcolor[HTML]{EAF3FA}    Generator           & \cellcolor[HTML]{EAF3FA}    74.52          & \cellcolor[HTML]{EAF3FA}    41.78          & \cellcolor[HTML]{EAF3FA}    72.07          & \cellcolor[HTML]{EAF3FA}    72.18               & \cellcolor[HTML]{EAF3FA}    56.39      & \cellcolor[HTML]{EAF3FA}      63.39    \\ 
                    & \cellcolor[HTML]{D0E7F7}                  & \cellcolor[HTML]{D0E7F7}    PredGen            & \cellcolor[HTML]{D0E7F7}    75.83           & \cellcolor[HTML]{D0E7F7}    43.22          & \cellcolor[HTML]{D0E7F7}    74.38          & \cellcolor[HTML]{D0E7F7}    72.90               & \cellcolor[HTML]{D0E7F7}    57.72    & \cellcolor[HTML]{D0E7F7}     64.81 \\ \hline

\multirow{12}{*}{\textbf{Gemma2-9B}} 
                    & \cellcolor[HTML]{ECEEFF}  MELoRA             & \cellcolor[HTML]{ECEEFF}    Predictor           & \cellcolor[HTML]{ECEEFF}    4.790           & \cellcolor[HTML]{ECEEFF}    1.064          & \cellcolor[HTML]{ECEEFF}    2.960          & \cellcolor[HTML]{ECEEFF}    2.850              & \cellcolor[HTML]{ECEEFF}    1.157  & \cellcolor[HTML]{ECEEFF}    2.560\\ 
                    & \cellcolor[HTML]{ECEEFF}             & \cellcolor[HTML]{ECEEFF}    Generator           & \cellcolor[HTML]{ECEEFF}    76.75           & \cellcolor[HTML]{ECEEFF}    40.96          & \cellcolor[HTML]{ECEEFF}    73.78          & \cellcolor[HTML]{ECEEFF}   72.49              & \cellcolor[HTML]{ECEEFF}    57.68    & \cellcolor[HTML]{ECEEFF}    64.33   \\ 
                    & \cellcolor[HTML]{D8DBF5}                    & \cellcolor[HTML]{D8DBF5}    PredGen            & \cellcolor[HTML]{D8DBF5}    78.32           & \cellcolor[HTML]{D8DBF5}    43.63          & \cellcolor[HTML]{D8DBF5}    74.65          & \cellcolor[HTML]{D8DBF5}    74.49               & \cellcolor[HTML]{D8DBF5}    59.32     & \cellcolor[HTML]{D8DBF5}    66.08 \\ 
                    & \cellcolor[HTML]{ECEEFF}    LoRA-FA           & \cellcolor[HTML]{ECEEFF}    Predictor           & \cellcolor[HTML]{ECEEFF}    3.640           & \cellcolor[HTML]{ECEEFF}    0.945           & \cellcolor[HTML]{ECEEFF}    3.230          & \cellcolor[HTML]{ECEEFF}    2.740               & \cellcolor[HTML]{ECEEFF}    1.570    & \cellcolor[HTML]{ECEEFF}      2.420  \\ 
                    & \cellcolor[HTML]{ECEEFF}                    & \cellcolor[HTML]{ECEEFF}    Generator           & \cellcolor[HTML]{ECEEFF}    74.79           & \cellcolor[HTML]{ECEEFF}    42.56  & \cellcolor[HTML]{ECEEFF}    74.24          & \cellcolor[HTML]{ECEEFF}    73.81              & \cellcolor[HTML]{ECEEFF}    57.93      & \cellcolor[HTML]{ECEEFF}     64.67  \\ 
                    & \cellcolor[HTML]{D8DBF5}                     & \cellcolor[HTML]{D8DBF5}    PredGen            & \cellcolor[HTML]{D8DBF5}    77.27          & \cellcolor[HTML]{D8DBF5}    43.86          & \cellcolor[HTML]{D8DBF5}    75.93          & \cellcolor[HTML]{D8DBF5}    74.63               & \cellcolor[HTML]{D8DBF5}    59.11     & \cellcolor[HTML]{D8DBF5}    66.16    \\ 
                    & \cellcolor[HTML]{ECEEFF}     MoSLoRA            & \cellcolor[HTML]{ECEEFF}    Predictor           & \cellcolor[HTML]{ECEEFF}    3.470           & \cellcolor[HTML]{ECEEFF}    0.976          & \cellcolor[HTML]{ECEEFF}    2.860          & \cellcolor[HTML]{ECEEFF}    2.430              & \cellcolor[HTML]{ECEEFF}    1.670     & \cellcolor[HTML]{ECEEFF}    2.280  \\ 
                    & \cellcolor[HTML]{ECEEFF}                    & \cellcolor[HTML]{ECEEFF}    Generator           & \cellcolor[HTML]{ECEEFF}    76.23   & \cellcolor[HTML]{ECEEFF}    43.73          & \cellcolor[HTML]{ECEEFF}    73.98          & \cellcolor[HTML]{ECEEFF}    72.68               & \cellcolor[HTML]{ECEEFF}    56.84      & \cellcolor[HTML]{ECEEFF}     64.69  \\ 
                    & \cellcolor[HTML]{D8DBF5}                       & \cellcolor[HTML]{D8DBF5}    PredGen            & \cellcolor[HTML]{D8DBF5}    77.36           & \cellcolor[HTML]{D8DBF5}    46.11          & \cellcolor[HTML]{D8DBF5}    75.85    & \cellcolor[HTML]{D8DBF5}    72.82               & \cellcolor[HTML]{D8DBF5}     58.77  & \cellcolor[HTML]{D8DBF5}    66.18   \\ 
                    & \cellcolor[HTML]{ECEEFF}    Propulsion              & \cellcolor[HTML]{ECEEFF}    Predictor           & \cellcolor[HTML]{ECEEFF}    2.750           & \cellcolor[HTML]{ECEEFF}    1.190   & \cellcolor[HTML]{ECEEFF}    2.660          & \cellcolor[HTML]{ECEEFF}    2.590     & \cellcolor[HTML]{ECEEFF}    1.280  & \cellcolor[HTML]{ECEEFF}    2.09         \\ 
                    & \cellcolor[HTML]{ECEEFF}                       & \cellcolor[HTML]{ECEEFF}    Generator           & \cellcolor[HTML]{ECEEFF}    75.23          & \cellcolor[HTML]{ECEEFF}    43.61          & \cellcolor[HTML]{ECEEFF}    73.72          & \cellcolor[HTML]{ECEEFF}    73.18               & \cellcolor[HTML]{ECEEFF}    57.52  & \cellcolor[HTML]{ECEEFF}    64.65\\ 
                    & \cellcolor[HTML]{D8DBF5}             & \cellcolor[HTML]{D8DBF5}    PredGen            & \cellcolor[HTML]{D8DBF5}    78.18           & \cellcolor[HTML]{D8DBF5}    44.83          & \cellcolor[HTML]{D8DBF5}    75.92          & \cellcolor[HTML]{D8DBF5}    72.87    & \cellcolor[HTML]{D8DBF5}    58.20   & \cellcolor[HTML]{D8DBF5}    66.00        \\ \hline
\multirow{12}{*}{\textbf{DeepSeek-R1-8B}} 
                    & \cellcolor[HTML]{F8EDEB}    MELoRA              & \cellcolor[HTML]{F8EDEB}    Predictor           & \cellcolor[HTML]{F8EDEB}    6.276          & \cellcolor[HTML]{F8EDEB}    1.785          & \cellcolor[HTML]{F8EDEB}    8.099    & \cellcolor[HTML]{F8EDEB}    5.374              & \cellcolor[HTML]{F8EDEB}    2.407   & \cellcolor[HTML]{F8EDEB}     4.791   \\ 
                    & \cellcolor[HTML]{F8EDEB}                       & \cellcolor[HTML]{F8EDEB}    Generator           & \cellcolor[HTML]{F8EDEB}    78.80          & \cellcolor[HTML]{F8EDEB}    45.25          & \cellcolor[HTML]{F8EDEB}    73.99          & \cellcolor[HTML]{F8EDEB}    71.89               & \cellcolor[HTML]{F8EDEB}    58.70      & \cellcolor[HTML]{F8EDEB}     65.95        \\ 
                    & \cellcolor[HTML]{F3DBD7}                   & \cellcolor[HTML]{F3DBD7}    PredGen            & \cellcolor[HTML]{F3DBD7}    80.99           & \cellcolor[HTML]{F3DBD7}    47.40    & \cellcolor[HTML]{F3DBD7}    74.89          & \cellcolor[HTML]{F3DBD7}    73.13               & \cellcolor[HTML]{F3DBD7}    59.15     & \cellcolor[HTML]{F3DBD7}     66.82       \\ 
                    & \cellcolor[HTML]{F8EDEB}     LoRA-FA           & \cellcolor[HTML]{F8EDEB}    Predictor           & \cellcolor[HTML]{F8EDEB}    6.285           & \cellcolor[HTML]{F8EDEB}    2.764   & \cellcolor[HTML]{F8EDEB}    9.564          & \cellcolor[HTML]{F8EDEB}    6.644   & \cellcolor[HTML]{F8EDEB}    2.364   & \cellcolor[HTML]{F8EDEB}     5.529    \\ 
                    & \cellcolor[HTML]{F8EDEB}                    & \cellcolor[HTML]{F8EDEB}    Generator           & \cellcolor[HTML]{F8EDEB}    77.92          & \cellcolor[HTML]{F8EDEB}    44.83          & \cellcolor[HTML]{F8EDEB}    73.70          & \cellcolor[HTML]{F8EDEB}    72.00               & \cellcolor[HTML]{F8EDEB}    59.67     & \cellcolor[HTML]{F8EDEB}    65.46     \\ 
                    & \cellcolor[HTML]{F3DBD7}                     & \cellcolor[HTML]{F3DBD7}    PredGen            & \cellcolor[HTML]{F3DBD7}    78.68          & \cellcolor[HTML]{F3DBD7}    47.46          & \cellcolor[HTML]{F3DBD7}    73.82          & \cellcolor[HTML]{F3DBD7}    73.55               & \cellcolor[HTML]{F3DBD7}    60.12  & \cellcolor[HTML]{F3DBD7}    66.71 \\ 
                    & \cellcolor[HTML]{F8EDEB}     MoSLoRA            & \cellcolor[HTML]{F8EDEB}    Predictor           & \cellcolor[HTML]{F8EDEB}    5.490           & \cellcolor[HTML]{F8EDEB}    1.882           & \cellcolor[HTML]{F8EDEB}    11.29           & \cellcolor[HTML]{F8EDEB}    5.005               & \cellcolor[HTML]{F8EDEB}    2.727  & \cellcolor[HTML]{F8EDEB}   5.286 \\ 
                    & \cellcolor[HTML]{F8EDEB}                     & \cellcolor[HTML]{F8EDEB}    Generator           & \cellcolor[HTML]{F8EDEB}    78.75           & \cellcolor[HTML]{F8EDEB}    45.71          & \cellcolor[HTML]{F8EDEB}    73.91          & \cellcolor[HTML]{F8EDEB}    72.73               & \cellcolor[HTML]{F8EDEB}    60.24    & \cellcolor[HTML]{F8EDEB}     66.93   \\ 
                    & \cellcolor[HTML]{F3DBD7}                      & \cellcolor[HTML]{F3DBD7}    PredGen            & \cellcolor[HTML]{F3DBD7}    80.00           & \cellcolor[HTML]{F3DBD7}    46.94          & \cellcolor[HTML]{F3DBD7}    76.01         & \cellcolor[HTML]{F3DBD7}    73.42               & \cellcolor[HTML]{F3DBD7}    61.15     & \cellcolor[HTML]{F3DBD7}     67.65  \\ 
                    & \cellcolor[HTML]{F8EDEB}     Propulsion        & \cellcolor[HTML]{F8EDEB}    Predictor           & \cellcolor[HTML]{F8EDEB}    5.333           & \cellcolor[HTML]{F8EDEB}    2.260          & \cellcolor[HTML]{F8EDEB}    8.837         & \cellcolor[HTML]{F8EDEB}    4.743               & \cellcolor[HTML]{F8EDEB}    1.930    & \cellcolor[HTML]{F8EDEB}    5.024  \\ 
                    & \cellcolor[HTML]{F8EDEB}                   & \cellcolor[HTML]{F8EDEB}    Generator           & \cellcolor[HTML]{F8EDEB}    76.80           & \cellcolor[HTML]{F8EDEB}    43.85         & \cellcolor[HTML]{F8EDEB}    73.75          & \cellcolor[HTML]{F8EDEB}    71.52              & \cellcolor[HTML]{F8EDEB}    60.03    & \cellcolor[HTML]{F8EDEB}     65.71  \\ 
                    & \cellcolor[HTML]{F3DBD7}                  & \cellcolor[HTML]{F3DBD7}    PredGen            & \cellcolor[HTML]{F3DBD7}    79.05           & \cellcolor[HTML]{F3DBD7}    45.37         & \cellcolor[HTML]{F3DBD7}    76.61          & \cellcolor[HTML]{F3DBD7}    73.68               & \cellcolor[HTML]{F3DBD7}    60.95   & \cellcolor[HTML]{F3DBD7}     68.34  \\ \hline

\end{tabular}}
\caption{Performance of LLMs with Different PEFT Methods Across the Commonsense Benchmarks. Here we used DeepSeek-R1-Distill-Llama-8B.}\label{tab:math_performance}
\end{table*}

\section{PEFT Methods}

Parameter-Efficient Fine-Tuning (PEFT) methods have gained significant attention for their ability to adapt large pre-trained models with minimal computational overhead, In this work we have used a whole range of PEFT methods. A key aspect of these methods lies in their reparameterization of the delta weights ($\Delta \mathbf{W}$), which represent the updates to the base model weights. Table~\ref{tab:delta_reparameterization} illustrates the diverse strategies employed by various PEFT methods, such as LoRA \cite{hu2021lora}, AdaLoRA \cite{zhang2023adalora}, RoCoFT \cite{kowsher2024rocoft}, DoRA \cite{freud1997dora}, MELoRA \cite{ren2024melora}, LoRA-FA \cite{zhang2023lora}, MoSLoRA \cite{wu2024mixture}, and Propulsion \cite{kowsher2024propulsion}. For instance, LoRA employs low-rank decomposition with $\mathbf{W}_{\text{down}}$ and $\mathbf{W}_{\text{up}}$ matrices, while AdaLoRA leverages singular value decomposition (SVD) for adaptive rank updates. RoCoFT introduces restricted row or column updates, and Propulsion focuses on updating only a mask $\mathbf{Z}$ while freezing the base weights. These approaches highlight the trade-offs between efficiency, flexibility, and performance in fine-tuning large models.

\begin{table*}[htp!]
\centering
\scalebox{0.75}{
\begin{tabular}{l|l|p{10cm}}
\hline
\textbf{Method}        & \textbf{$\Delta W$ Reparameterization} & \textbf{Notes} \\ \hline
LoRA                   & $\Delta W = W_{\text{down}} W_{\text{up}}$  & $W_{\text{down}} \in \mathbb{R}^{d \times r}$, $W_{\text{up}} \in \mathbb{R}^{r \times d}$, and $r \ll \{k,d\}$. \\ \hline
AdaLoRA                & $\Delta W  = PAQ$ & $PP^\top = P^\top P \neq I = QQ^\top = Q^\top Q$, $\Lambda = \text{diag}(\sigma_1, \sigma_2, \dots, \sigma_r)$. \\ \hline
RoCoFT                 & $\Delta W = W_0 + R$ or $\Delta W = W_0 + C$ & $R$ and $C$ are restricted weight matrices such that only at most $r$ of the rows or columns are non-zero. \\ \hline
DoRA                   & $\Delta W = W_{\text{down}} W_{\text{up}}$ & Similar to LoRA but with dynamic rank adaptation during training. \\ \hline
MELoRA                 & $\Delta W = W_{\text{down}} W_{\text{up}}$ & Multi-expert LoRA, where multiple low-rank updates are combined. \\ \hline
LoRA-FA                & $\Delta W = W_{\text{down}} W_{\text{up}} = Q R W_{\text{up}}$ & $W_{\text{down}}$ is frozen, and only $W_{\text{up}}$ is updated. \\ \hline
MoSLoRA                & $\Delta W = W_{\text{down}} W_{\text{up}}$ & Mixture of sparse LoRA, combining sparse and low-rank updates. \\ \hline
Propulsion             & $\Delta W = W \odot Z$ & $W$ is frozen, and only $Z$ is updated. \\ \hline
\end{tabular}
}
\caption{Comparison of delta weight reparameterization across various PEFT methods.} \label{tab:delta_reparameterization}
\end{table*}

\section{Token Laval Mutual Information} \label{append:MI}

\subsection{MI for Classification}  
To further analyze mutual information (MI), we present token-label MI heatmaps in Figures~\ref{fig:MI_cls0} and \ref{fig:MI_cls1} using the PIQA dataset. In Figure~\ref{fig:MI_cls0}, the correct answer is "Solution 1," where the input text has higher mutual information with "Solution 1" than with "Solution 2." Similarly, in Figure~\ref{fig:MI_cls1}, the correct answer is "Solution 2," and the input shows stronger MI with "Solution 2" than with "Solution 1."  

\subsection{MI for Regression}  
In Figures~\ref{fig:MI_reg0}, \ref{fig:MI_reg3}, and \ref{fig:MI_reg2}, we present MI heatmaps showing the relationship between input tokens and regression values from the LCP dataset.

\begin{figure*}[!t]
    \centering
    \includegraphics[width=1.00\linewidth]{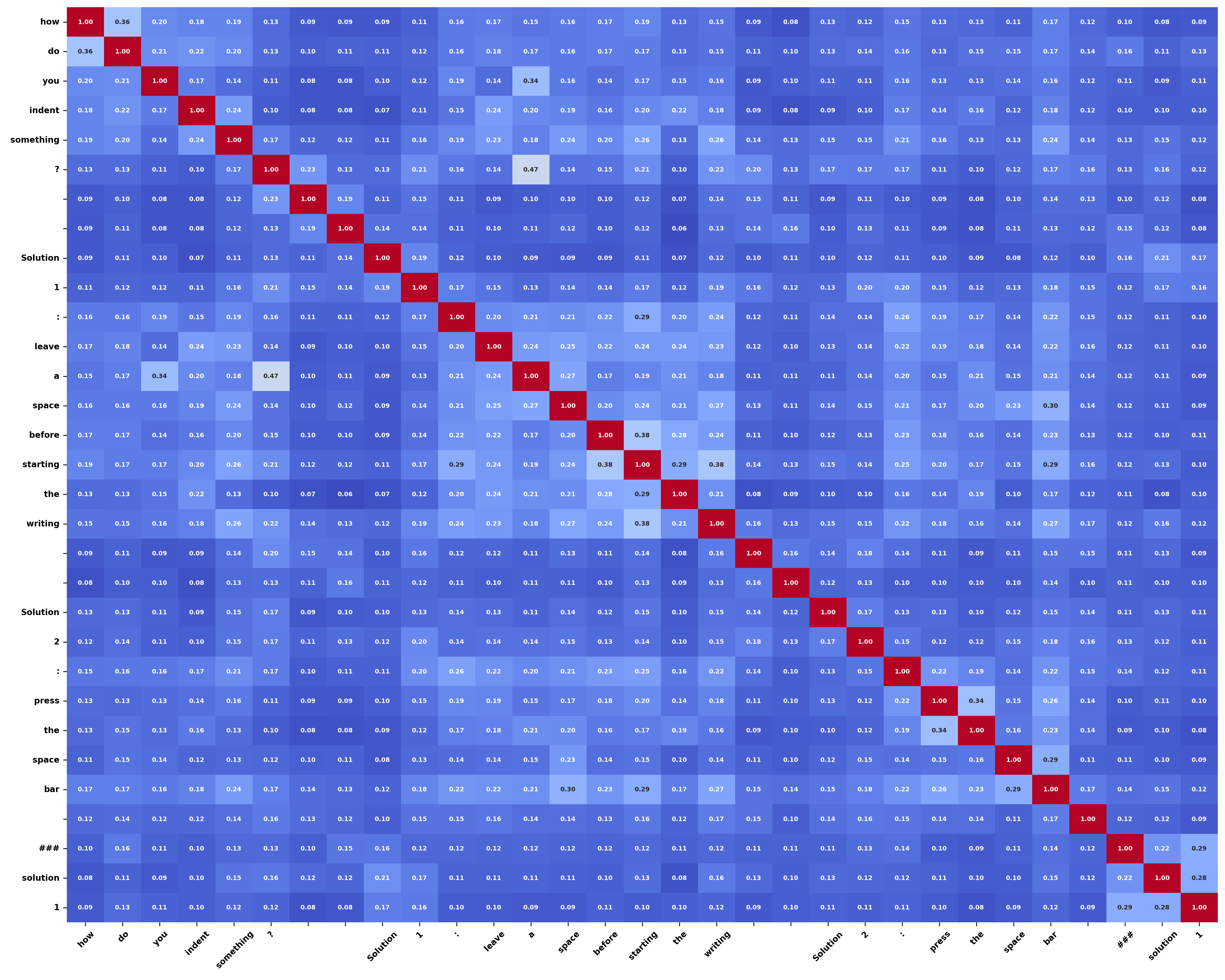}
\caption{Mutual information (MI) heatmap for classification on the PIQA dataset. The correct answer is "Solution 1," which has higher MI with the input text compared to "Solution 2," indicating stronger alignment with the correct label.}

\label{fig:MI_cls0}
\end{figure*}

\begin{figure*}[!t]
    \centering
    \includegraphics[width=1.00\linewidth]{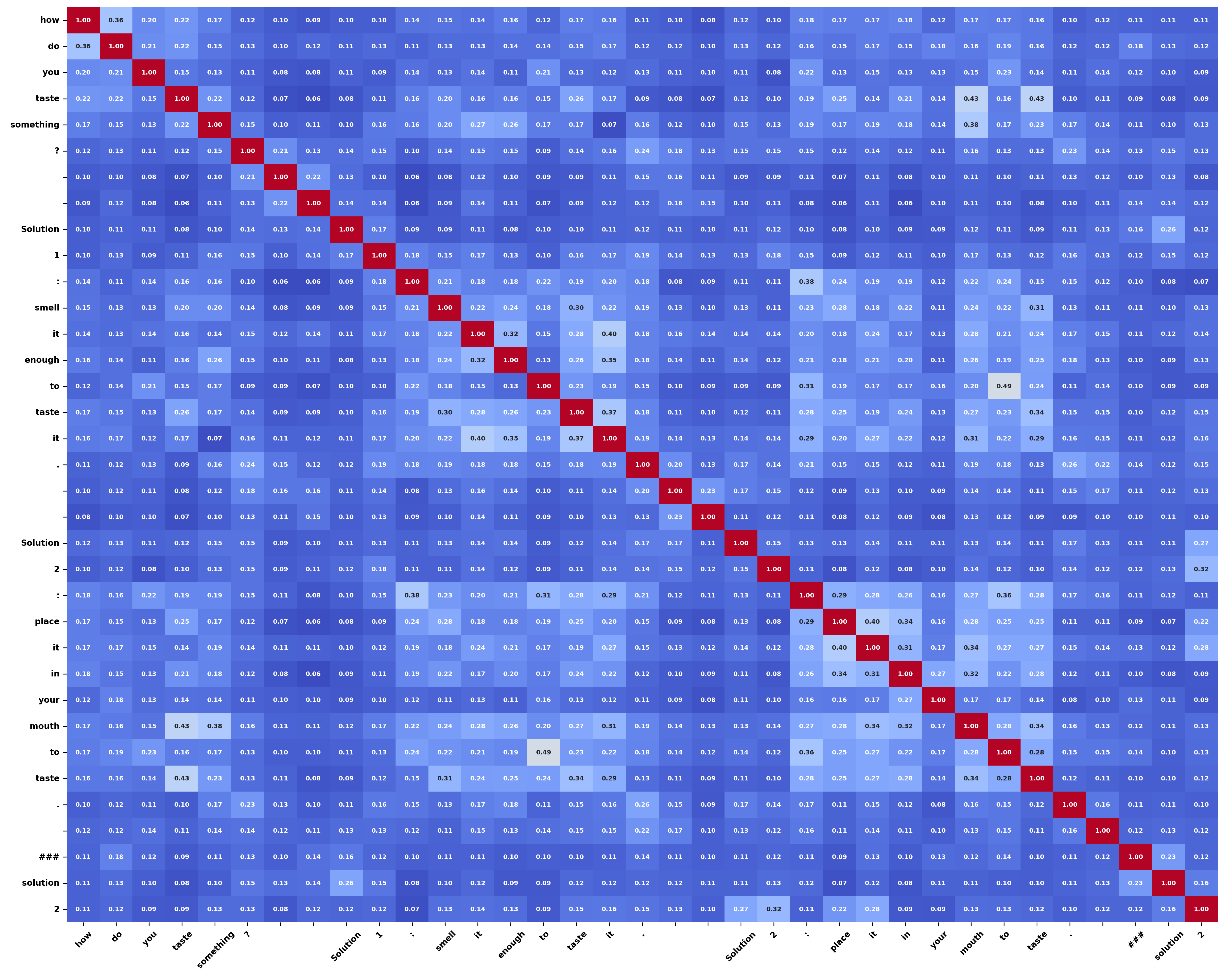}
\caption{Mutual information (MI) heatmap for classification on the PIQA dataset. The correct answer is "Solution 2," which has higher MI with the input text compared to "Solution 1," demonstrating effective label association.}

\label{fig:MI_cls1}
\end{figure*}

\begin{figure*}[!t]
    \centering
    \includegraphics[width=1.0\linewidth]{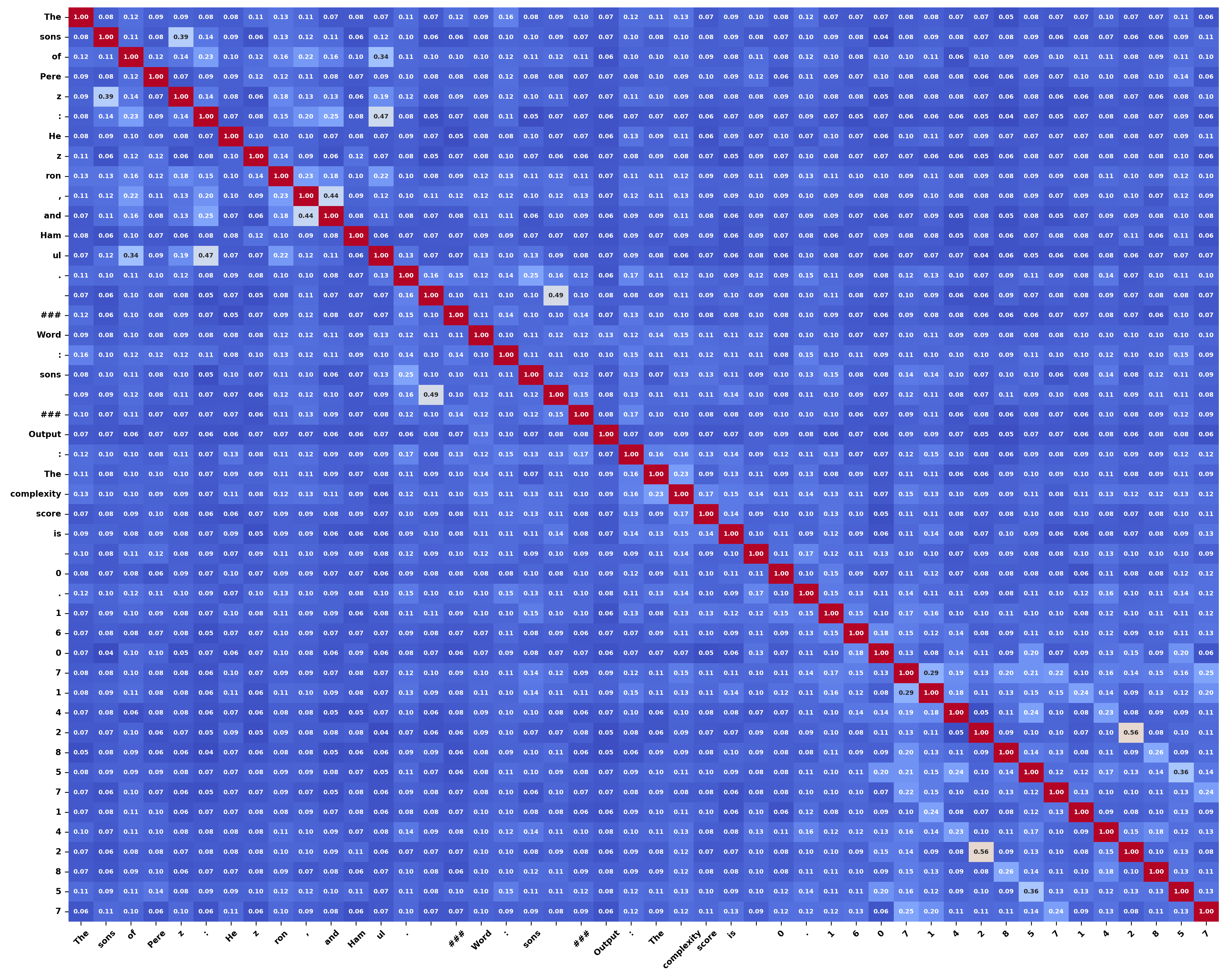}

\caption{Mutual information (MI) heatmap for regression on the LCP dataset. The visualization shows how input tokens contribute to the predicted regression values, highlighting key dependencies.}

\label{fig:MI_reg0}
\end{figure*}

\begin{figure*}[!t]
    \centering
    \includegraphics[width=1.0\linewidth]{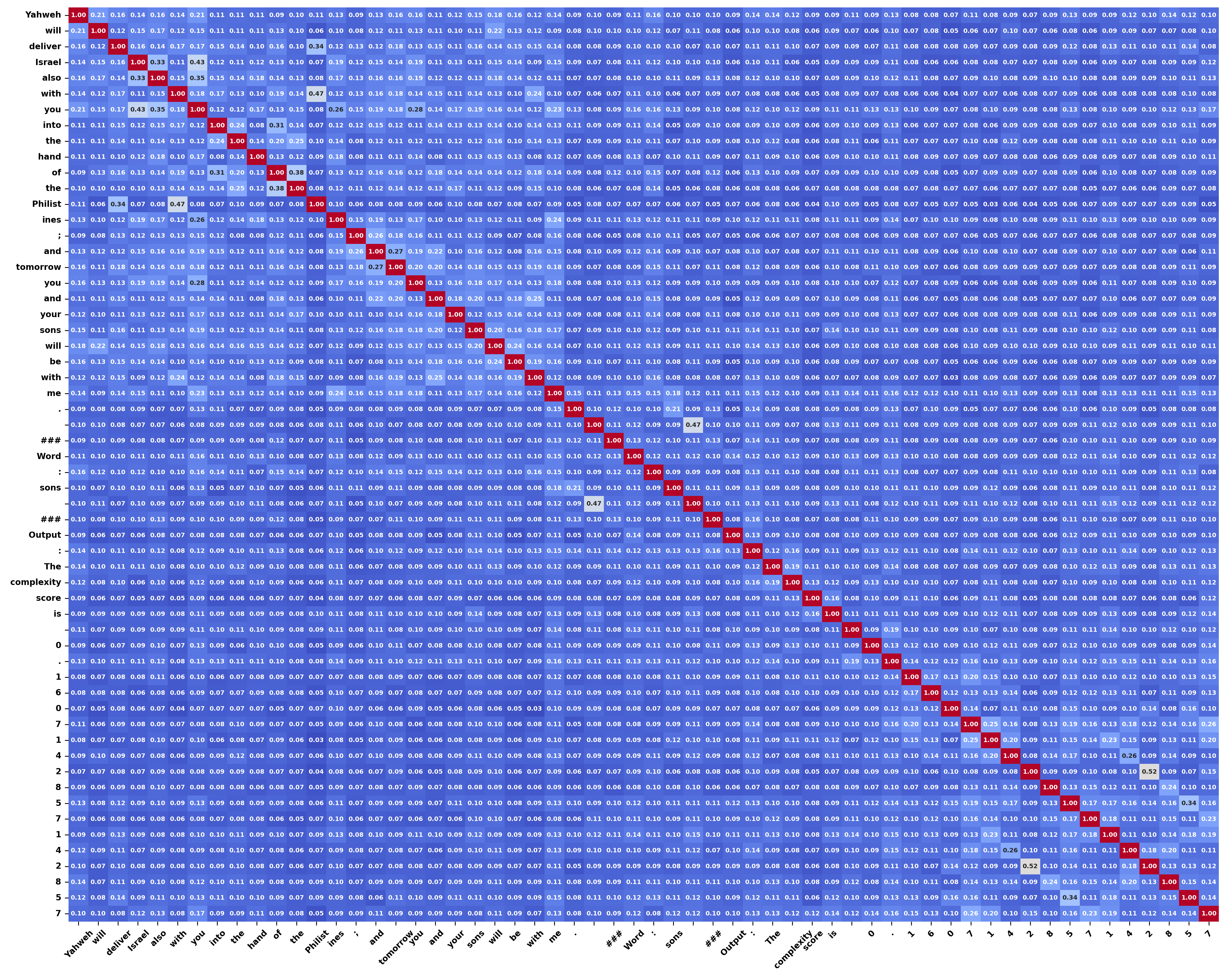}

\caption{Mutual information (MI) heatmap for regression on the LCP dataset. The visualization shows how input tokens contribute to the predicted regression values, highlighting key dependencies.}

\label{fig:MI_reg1}
\end{figure*}

\begin{figure*}[!t]
    \centering
    \includegraphics[width=1.0\linewidth]{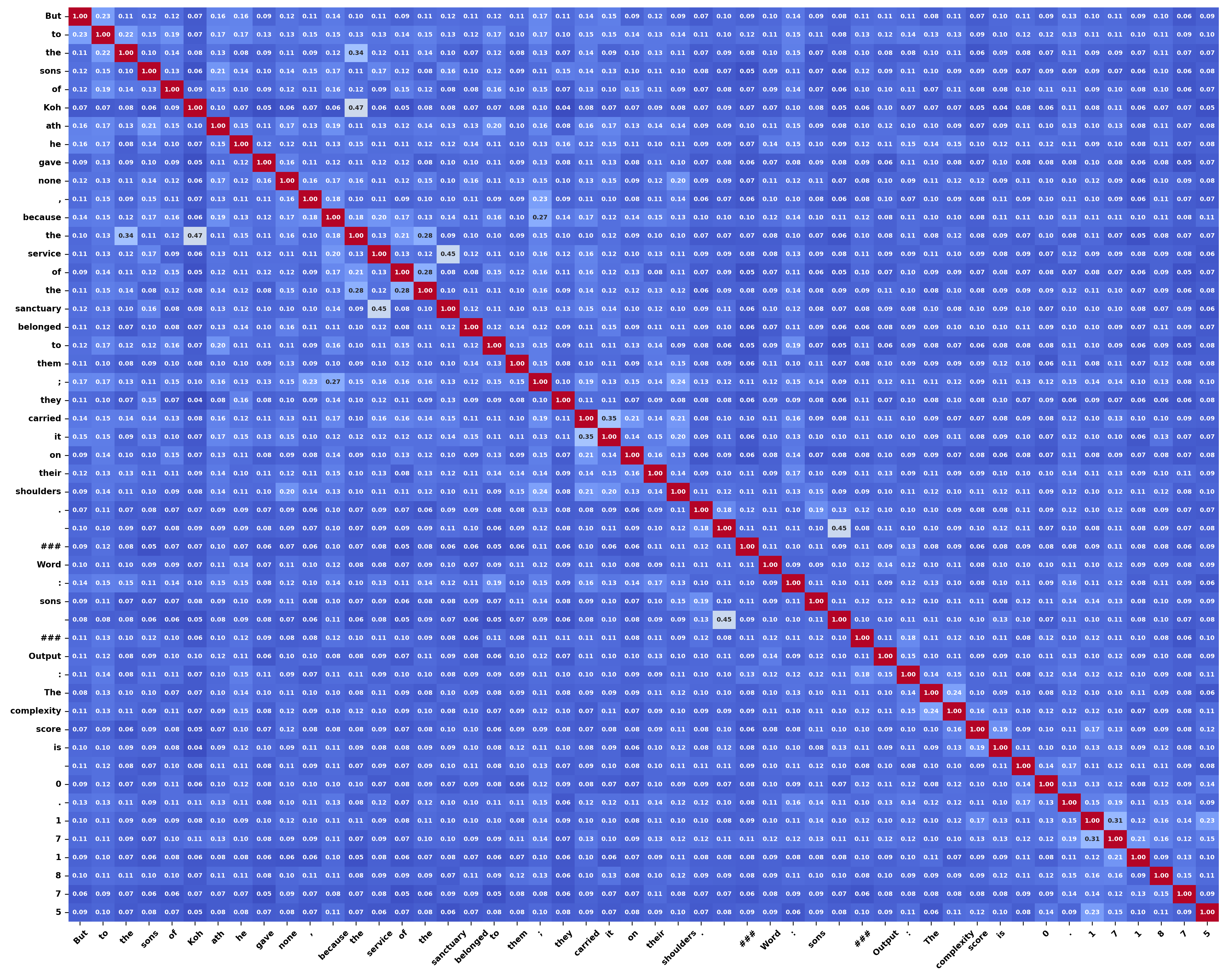}

\caption{Mutual information (MI) heatmap for regression on the LCP dataset. The model learns fine-grained relationships between input tokens and regression outputs, preserving important task-specific information.}

\label{fig:MI_reg2}
\end{figure*}

\begin{figure*}[!t]
    \centering
    \includegraphics[width=1.0\linewidth]{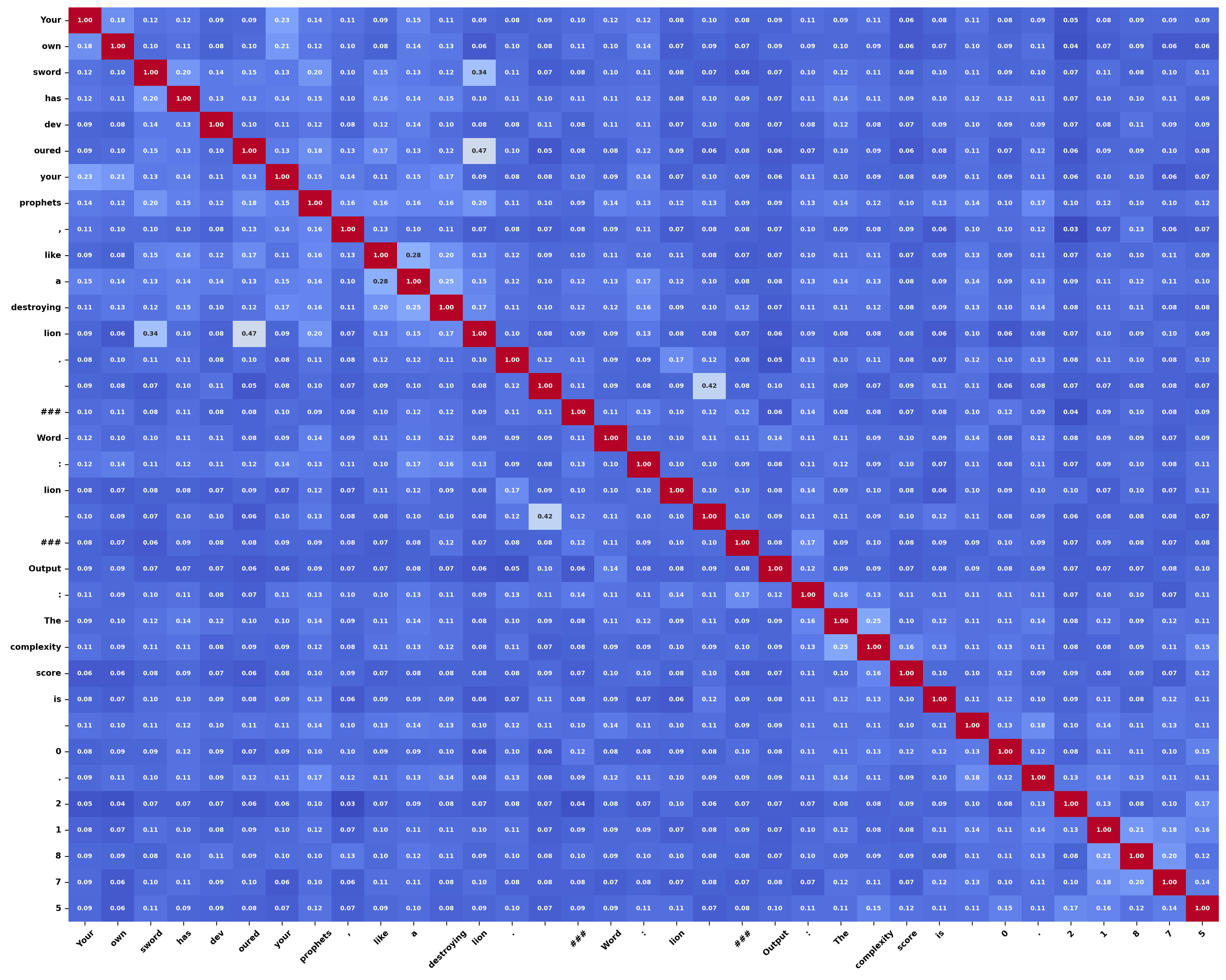}
\caption{Mutual information (MI) heatmap for regression on the LCP dataset. The results demonstrate that token-level generation retains more mutual information with regression values compared to pooled representations.}

\label{fig:MI_reg3}
\end{figure*}

\end{document}